\PassOptionsToPackage{numbers, compress}{natbib}
\documentclass{article}

\newcommand{\treportflag}{1} 

\ifnum\treportflag=1
  \usepackage[preprint]{neurips_2026}
\else
  \usepackage{neurips_2026}
\fi

\usepackage[utf8]{inputenc} 
\usepackage[T1]{fontenc}    
\usepackage{hyperref}       
\usepackage{url}            
\usepackage{booktabs}       
\usepackage{amsfonts}       
\usepackage{nicefrac}       
\usepackage{xcolor}         
\usepackage{microtype}
\usepackage{graphicx}
\usepackage{subcaption}
\usepackage{drago}

\newtheorem*{theorem-restate}{Theorem \ref{thm:cge-bounds}}

\title{Causal Fairness for Survival Analysis}

\author{%
  Drago Ple{\v c}ko \\
  Department of Statistics \& Data Science \\ 
  UCLA
}

\usepackage{selectp}

\begin{document}
\ifnum\treportflag=1
\fi

\maketitle

\begin{abstract}
    In the data-driven era, large-scale datasets are routinely collected and analyzed using machine learning (ML) and artificial intelligence (AI) to inform decisions in high-stakes domains such as healthcare, employment, and criminal justice, raising concerns about the fairness behavior of these systems. Existing works in fair ML cover tasks such as bias detection, fair prediction, and fair decision-making, but largely focus on static settings. At the same time, fairness in temporal contexts, particularly survival/time-to-event (TTE) analysis, remains relatively underexplored, with current approaches to fair survival analysis adopting statistical fairness definitions, which, even with unlimited data, cannot disentangle the causal mechanisms that generate disparities. To address this gap, we develop a causal framework for fairness in TTE analysis, enabling the decomposition of disparities in survival into contributions from direct, indirect, and spurious pathways. This provides a human-understandable explanation of why disparities arise and how they evolve over time. Our non-parametric approach proceeds in four steps: (1) formalizing the necessary assumptions about censoring and lack of confounding using a graphical model; (2) recovering the conditional survival function given covariates; (3) applying the Causal Reduction Theorem to reframe the problem in a form amenable to causal pathway decomposition; (4) estimating the effects efficiently. Finally, our approach is used to analyze the temporal evolution of racial disparities in outcome after admission to an intensive care unit (ICU).
\end{abstract}

\section{Introduction} \label{sec:intro}
Large-scale data collection and analysis is increasingly changing decision-making in a variety of real-world settings. Datasets capturing information from hiring processes, university admissions, law enforcement, credit lending and loan approvals, health care interventions, and other high-stakes domains are now routinely processed using machine learning and artificial intelligence to inform or automate decisions \citep{khandani2010consumer,mahoney2007method,brennan2009evaluating, shaheen2021applications}. 
In this context, society is increasingly concerned about how the nature, scope, and quality of these datasets, combined with algorithmic analysis, may shape the future world in which existing decision processes are automated. 
Prior works highlight that data-driven systems can perpetuate or even amplify inequities between demographic groups, with documented examples in criminal justice \citep{ProPublica}, computer vision \citep{pmlr-v81-buolamwini18a}, and online advertising \citep{Sweeney13,Datta15}, to cite a few. More recently, similar concerns have emerged in generative AI \citep{gallegos2024bias}.
Notably, however, the new data-driven era has also highlighted the fact that inequities are pervasive when decisions are made by humans, and the datasets recording these decisions often encode historical biases. 
Well-known examples in the literature include the gender pay gap \citep{blau1992gender,blau2017gender} and racial bias in criminal sentencing \citep{sweeney1992influence,pager2003mark}. 
As a result, data collected from the current reality will encode these patterns, effectively embedding past discriminatory decisions against certain protected groups. Naturally, predictive or generative models developed on such data are at risk of exhibiting undesirable bias. These observations gave rise to the growing field of fair machine learning, although as described above, many core issues originate before ML systems are even deployed.

Within this context, it is useful to distinguish between different tasks appearing in the growing literature on fair ML. One can identify three specific and distinct tasks: (1) bias detection and quantification for existing outcomes or decision policies; (2) construction of fair predictions of an outcome; and (3) construction of fair decision-making policies intended for real-world implementation \citep{plecko2022causal}. An additional helpful distinction is between static and dynamic fairness settings, which interacts with the above-mentioned tasks. In the former, the goal is to analyze a \textit{snapshot} of reality at a single time point. In the latter, the objective is to understand how disparities between groups emerge and evolve over time. Dynamic, time-resolved settings arise naturally when institutions repeatedly evaluate cohorts, such as in hiring or college admission decisions. In some cases, the same cohort of individuals may be subject to sequential decisions, or the focus may be on the time elapsed before an event occurs, with attention to group differences. This final case, involving time-to-event (TTE) or survival analysis, is the focus of this paper.

While a flurry of recent work focuses on fairness \citep{pessach2022review}, most commonly in static settings concerning fair prediction, some areas in the fair ML landscape remain relatively understudied, particularly those involving temporal data. Existing works on fairness in TTE analysis primarily adopt statistical definitions of fairness \citep{sonabend2022flexible, rahman2022fair, hu2024fairness, xie2025fairness}, which, on their own, cannot distinguish between different causal mechanisms that generate disparity in the real world, even with unlimited data. For this reason, a field studying fairness through a causal lens has emerged \citep{kusner2017counterfactual, zhang2018fairness, chiappa2019path, plecko2022causal}, enabling human-understandable and interpretable definitions and metrics of fairness that are explicitly tied to the causal mechanisms transmitting change between groups. Within the literature on causal fairness, however, few dynamic settings have been investigated, with some notable exceptions \citep{nabi2019learning, creager2020causal}. Thus, within the causal approach, there remains a need for a deeper understanding of how causal reasoning can be applied to fairness in temporal data, such as TTE analysis.

In this work, our aim is to fill this gap, and offer a principled approach for causal fairness analysis in TTE settings. We next describe the motivating example studied in the paper:
\begin{figure}[t]
    \centering
    \begin{subfigure}[b]{0.525\textwidth}
        \centering
        \includegraphics[width=0.95\linewidth]{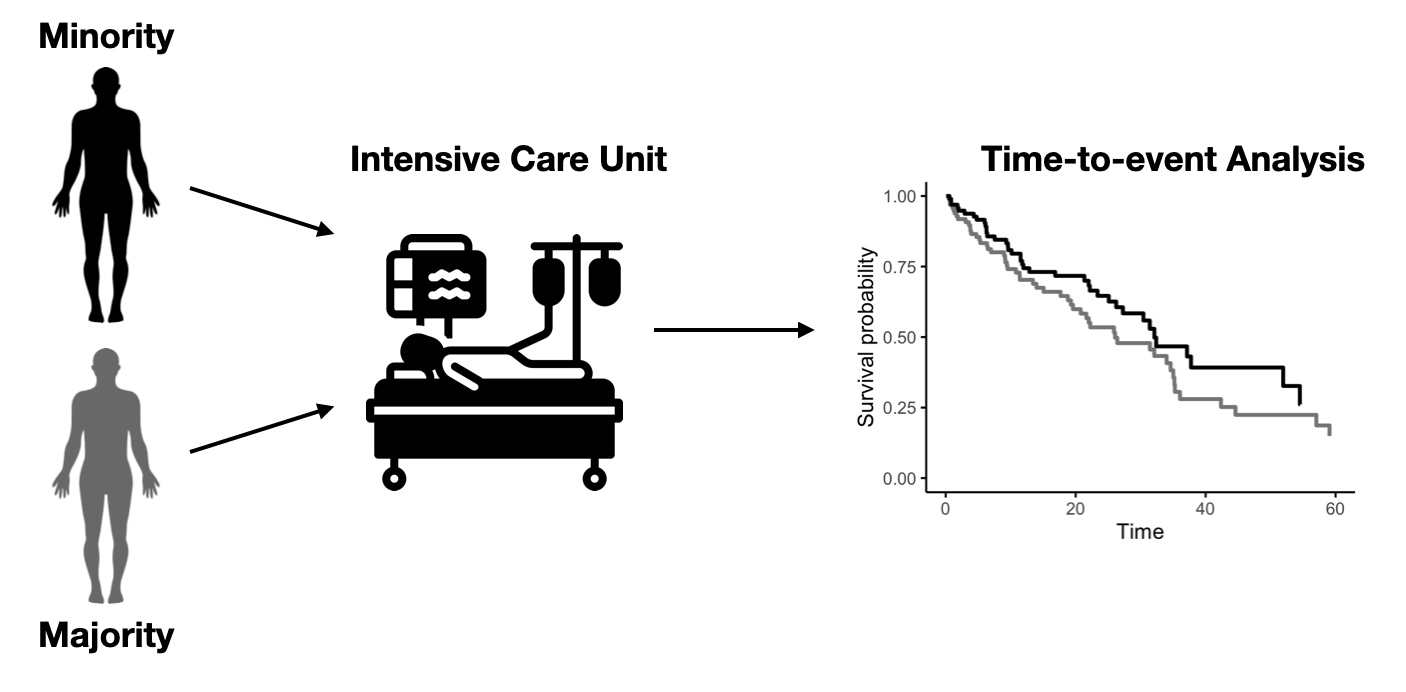}
        \vspace{-0.15in}
        \caption{Survival analysis after ICU admission.}
        \label{fig:intro-ex-vis}
    \end{subfigure}
    \begin{subfigure}[b]{0.465\textwidth}
        \centering
        \scalebox{0.75}{
        \begin{tikzpicture}
	 	[>=stealth, rv/.style={thick}, rvc/.style={triangle, draw, thick, minimum size=8mm}, node distance=7mm]
	 	\pgfsetarrows{latex-latex};
	 	\begin{scope}
	 	\node[rv] (c) at (2.2,1.47) {age};
            \node[rv] (c) at (3,1.5) {sex};
            \node[rv] (c) at (3.8,1.525) {SES};
	 	\node[rv] (a) at (0,0) {race};
	 	\node[rv, align=center] (m) at (2.2,-1.5) {admission\\diagnosis};
	 	\node[rv, align=center] (r) at (3.8,-1.5) {illness\\severity};
	 	\node[rv] (y) at (6.5,1) {time-to-death};
            \node[rv] (y2) at (6.5,-1) {censoring};
            \node[rv] (empty) at (0,-2.2) {};
	 	
	 	\node (Zset) [draw,rectangle,minimum width=2.5cm,minimum height=0.8cm,label={[above right]Confounders $Z$}] at (3,1.5) {};
	 	\node (Wset) [draw,rectangle,minimum width=3.5cm,minimum height=1cm, label={[above left]Mediators $W$}] at (3,-1.5) {};
	 	
	 	\path[->] (a) edge[bend left = 0] (y);
            \path[->] (a) edge[bend left = 0] (y2);
	 	\path[->] (a) edge[bend left = -20] (Wset);
	 	\path[->] (Zset) edge[bend left = 0] (Wset);
	 	\path[->] (Wset) edge[bend left = 0] (y);
            \path[->] (Wset) edge[bend left = 0] (y2);
	 	\path[->] (Zset) edge[bend left = 0] (y);
            \path[->] (Zset) edge[bend left = 0] (y2);
	 	
	 	\path[<->,dashed] (a) edge[bend left = 20](Zset);
	 	\end{scope}
	\end{tikzpicture}
        }
        \vspace{-0.1in}
        \caption{Causal diagram.}
        \label{fig:intro-sfm}
    \end{subfigure}
    \caption{Introductory example of race differentials after intensive care unit (ICU) admission.}
    \label{fig:intro-ex}
\end{figure}
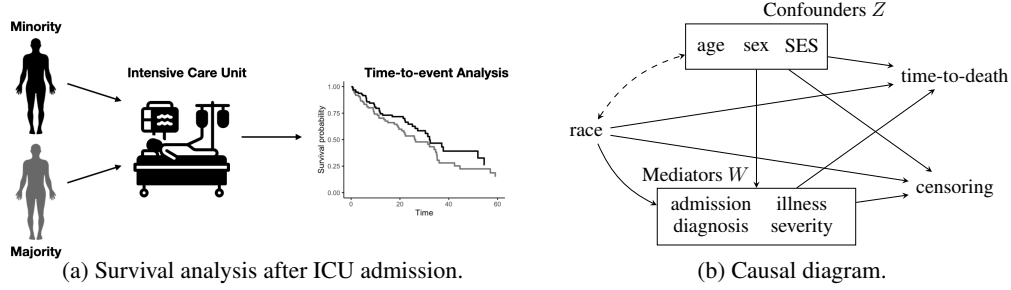
\begin{example}[Time-to-Death Differentials after ICU]
    Patients are admitted to an intensive care unit (ICU) for various life-threatening conditions. Upon admission, patient data is collected, including admission diagnosis and illness severity information ($W$), as well as demographic information such as age and socioeconomic status ($Z$). The outcome of interest is the time to death after ICU admission ($T$), and the study aims to compare groups defined by race ($X$).
    \cref{fig:intro-ex-vis} provides an overview.
    
    In the data, instead of directly observing the true time to death $T$, we observe the patient’s study end time $M$ and a censoring indicator $\delta$. If $\delta = 0$, $M$ corresponds to a censoring time (the patient is removed from the study and no longer observed), while $\delta = 1$ implies that $M$ equals the death time $T$. 
    The exact censoring mechanism depends on the study design and data availability. 

    The causal diagram in \cref{fig:intro-sfm} illustrates that the attribute $X$ may influence the time to death $T$ through multiple causal pathways: direct path ($X \to T$), indirect path ($X \to W \to T$), and spurious (confounded) path ($X \bidir Z \to T$). While a typical static analysis might examine the direct, indirect, and spurious effects of race on in-hospital mortality (or mortality within a fixed time window), our goal is to explain, in a time-dependent fashion, how disparities in survival evolve.
\end{example}
Specifically, the contributions of the paper are the following:
\begin{enumerate}[label=(\roman*)]
    \item We formalize causal fairness in survival analysis and prove the Causal Reduction Theorem (\cref{thm:causal-reduction}), allowing us to identify and decompose group disparities in TTE settings.
    \item We instantiate the framework in three regimes: non-informative censoring (\cref{sec:classic}), competing risks (\cref{sec:cr}), and informative censoring (\cref{sec:ic}). We derive identification results for each of the settings (\cref{prop:id}).
    \item We prove two new technical results: a doubly robust estimator for censored data (\cref{thm:dr-est}), and sharp bounds for copula-graphic estimation under non-disjoint increments (\cref{thm:cge-bounds}).
    \item We analyze race differentials in post-ICU mortality, demonstrating that causal fairness for TTE yields scientific insights missed by standard statistical fairness analyses.
\end{enumerate}

\subsection{Related Work}
Some previous works investigate fairness notions for TTE analysis \citep{sonabend2022flexible, rahman2022fair, hu2024fairness}. These studies focus on statistical definitions of fairness for quantifying disparities, whereas our goal is to develop and apply causal notions of fairness. Our work is also connected to the literature on causal fairness \citep{zhang2018fairness, plecko2022causal} that introduces important causal fairness notions, but is concerned with the static setting only, while we extend these ideas to the context of TTE analysis. 
Another work related to ours is \citep{pham2025boundaries}, which provides a causal-inspired fairness discussion in a TTE setting, although a statistical notion of fairness is adopted therein.
Finally, our work is related to the literature on causal mediation for survival analysis \citep{lange2011direct, vanderweele2011causal}. However, most of this literature adopts parametric models, often specifying particular forms for the survival time or mediator distributions. By contrast, we focus on non-parametric inference, avoiding strong modeling assumptions about the distributional form of survival times or mediators.

\subsection{Preliminaries} \label{sec:prelims}
We use structural causal models (SCMs) as our semantical framework \citep{pearl:2k}. An SCM is
a tuple $\mathcal{M} := \langle V, U, \mathcal{F}, P(u)\rangle$ , where $V$, $U$ are sets of
endogenous (observable) and exogenous (latent) variables 
respectively, $\mathcal{F}$ is a set of functions $f_{V_i}$,
one for each $V_i \in V$, where $V_i \gets f_{V_i}(\pa(V_i), U_{V_i})$ for some $\pa(V_i)\subseteq V$ and
$U_{V_i} \subseteq U$. $P(u)$ is a strictly positive probability measure over $U$. Each SCM $\mathcal{M}$ is associated with a causal diagram $\mathcal{G}$ over the nodes $V$, where $V_i \rightarrow V_j$ if $V_i$ is an argument of $f_{V_j}$, and $V_i \bidir V_j$ if the corresponding $U_{V_i}, U_{V_j}$ are not independent. An instantiation of the exogenous $U = u$ is called a \textit{unit}. By $Y_{x}(u)$ we denote the potential response of $Y$ when setting $X=x$ for the unit $u$, which is the solution for $Y(u)$ to the set of equations obtained by evaluating the unit $u$ in the submodel $\mathcal{M}_x$, in which all equations in $\mathcal{F}$ associated with $X$ are replaced by $X = x$.
\section{Causally Fair Survival -- Non-Informative Censoring} \label{sec:classic}
Our first step is to conceptualize the notions of causal fairness in the context of survival analysis.
We assume access to a specific cluster causal diagram $\mathcal{G}_{\text{SFM}}$ known as the standard fairness model (SFM) \citep{plecko2022causal}, which we adapt to the setting of survival analysis. The SFM (Fig.~\ref{fig:sfm-nic}) consists of the following: \textit{protected attribute}, labeled $X$ (e.g., gender, race, religion), assumed to be binary; the set of \textit{confounding} variables $Z$, which are not causally influenced by the attribute $X$ (e.g., demographic information, zip code), but could share co-variations with $X$; the set of \textit{mediator} variables $W$ that are possibly causally influenced by the attribute. 

Further, the variable $T$ denotes the survival time of the individual for the specific event of interest, while $C$ denotes the time to censoring. These variables are in a shaded gray area, since they are assumed not to be observed. Instead, we observe the pair $(M, \delta)$, where $M = \min\{T, C\}$ and $\delta = \mathbb{1}(T \leq C)$ is the indicator of whether the observed event is death or not. 

The SFM in Fig.~\ref{fig:sfm-nic} encodes some important assumptions for our analysis, which can be broken into two parts. Firstly, the graphical model implies that variables $T, C$ are independent conditional on $X, Z, W$, written as $T \ci C \mid X, Z, W$, since there is no unobserved confounding between $T$ and $C$ (absence of a bidirected arrow $T \bidir C$). This assumption, which stipulates that the censoring mechanism is (conditionally) independent of the survival mechanism, allows the application of the standard tools of survival analysis, such as the (conditional) Kaplan-Meier \cite{kaplan1958nonparametric} or Nelson-Aalen \citep{nelson1972theory, aalen1978nonparametric} estimators. This assumption requires careful justification, and can often be used when the censoring occurs due to administrative reasons such as study termination (violations of this assumption are discussed in \cref{sec:ic}). 
The second part of the assumptions relates to the lack-of-confounding assumptions typically used in the causal inference literature. In particular, note that there are no bidirected edges between sets $\{T, C, M, \delta\}$ and $\{X, Z, W\}$, and also there are no bidirected arrows between $X, W$ or $Z, W$. These lack-of-confounding assumptions (encoded through absences of bidirected edges) are important for identification of different effects of the attribute $X$ on the survival curve. 
The above two sets of assumptions also correspond to the steps of our analysis in the sequel. 

\begin{figure*}[t]
  \centering
  \newcommand{\sfmscale}{0.83}
  \begin{subfigure}{0.25\textwidth}
    \centering
    \scalebox{\sfmscale}{    \begin{tikzpicture}
	 [>=stealth, rv/.style={thick}, rvc/.style={triangle, draw, thick, minimum size=7mm}, node distance=18mm]
	 \pgfsetarrows{latex-latex};
	 \begin{scope}
		\node[rv] (0) at (0,1) {$Z$};
	 	\node[rv] (1) at (-1.25,0) {$X$};
	 	\node[rv] (2) at (0,-1) {$W$};
	 	\node[rv] (3) at (1.25,0.6) {$T$};
            \node[rv] (4) at (1.25,-0.6) {$C$};
            \begin{scope}[on background layer]
              \node[draw, gray, fill=gray!5, fit=(3)(4), inner sep=-0.75pt, rounded corners] (unobs) {};
            \end{scope}
            \node[rv] (5) at (2.25,0) {$(M, \delta)$};
	 	\draw[->] (1) -- (2);
		\draw[->] (0) -- (3);
            \draw[->] (0) -- (4);
	 	\path[->] (1) edge[bend left = 0] (3);
            \path[->] (1) edge[bend left = 0] (4);
		\path[<->] (1) edge[bend left = 30, dashed] (0);
	 	\draw[->] (2) -- (3);
            \draw[->] (2) -- (4);
		\draw[->] (0) -- (2);
            \draw[->] (3) -- (5);
            \draw[->] (4) -- (5);
	 \end{scope}
	 \end{tikzpicture}}
    \caption{Non-Informative Case.}
    \label{fig:sfm-nic}
  \end{subfigure}\hfill
  \begin{subfigure}{0.2\textwidth}
    \centering
    \scalebox{\sfmscale}{\begin{tikzpicture}
	 [>=stealth, rv/.style={thick}, rvc/.style={triangle, draw, thick, minimum size=7mm}, node distance=18mm]
	 \pgfsetarrows{latex-latex};
	 \begin{scope}
		\node[rv] (0) at (0,1) {$Z$};
	 	\node[rv] (1) at (-1.25,0) {$X$};
	 	\node[rv] (2) at (0,-1) {$W$};
	 	\node[rv] (3) at (1.5,0) {$\Phi(t)$};
	 	\draw[->] (1) -- (2);
		\draw[->] (0) -- (3);
	 	\path[->] (1) edge[bend left = 0] (3);
		\path[<->] (1) edge[bend left = 30, dashed] (0);
	 	\draw[->] (2) -- (3);
		\draw[->] (0) -- (2);
	 \end{scope}
	 \end{tikzpicture}}
    \caption{Thm.~\ref{thm:causal-reduction} Reduction.}
    \label{fig:sfm-reduce}
  \end{subfigure}\hfill
  \begin{subfigure}{0.27\textwidth}
    \centering
    \scalebox{\sfmscale}{        \begin{tikzpicture}
	 [>=stealth, rv/.style={thick}, rvc/.style={triangle, draw, thick, minimum size=7mm}, node distance=18mm]
	 \pgfsetarrows{latex-latex};
	 \begin{scope}
		\node[rv] (0) at (0,1) {$Z$};
	 	\node[rv] (1) at (-1.25,0) {$X$};
	 	\node[rv] (2) at (0,-1) {$W$};

	 	\node[rv] (3) at (1.25,0.9) {$T_1$};
        \node[rv] (6) at (1.2,0.5) {$\vdots$};
        \node[rv] (7) at (1.25,-0.2) {$T_K$};
        \node[rv] (4) at (1.25,-0.9) {$C$};

            \begin{scope}[on background layer]
              \node[draw, gray, fill=gray!5, fit=(3)(7)(4), inner sep=-0.75pt, rounded corners] (unobs) {};
            \end{scope}

            \node[rv] (5) at (2.5,-0.2) {$(M, \delta)$};

	 	\draw[->] (1) -- (2);
		\draw[->] (0) -- (2);

		\draw[->] (0) -- (3);
		\draw[->] (2) -- (3);
	 	\path[->] (1) edge[bend left = 0] (3);

		\draw[->] (0) -- (7);
		\draw[->] (2) -- (7);
	 	\path[->] (1) edge[bend left = 0] (7);

            \draw[->] (0) -- (4);
            \draw[->] (2) -- (4);
            \path[->] (1) edge[bend left = 0] (4);

		\path[<->] (1) edge[bend left = 30, dashed] (0);

        \path[<->] (3) edge[bend left = 25, dashed] (7);

            \draw[->] (3) -- (5);
            \draw[->] (7) -- (5);
            \draw[->] (4) -- (5);

	 \end{scope}
	 \end{tikzpicture}}
    \caption{Competing Risks.}
    \label{fig:sfm-cr}
  \end{subfigure}\hfill
  \begin{subfigure}{0.25\textwidth}
    \centering
    \scalebox{\sfmscale}{\begin{tikzpicture}
	 [>=stealth, rv/.style={thick}, rvc/.style={triangle, draw, thick, minimum size=7mm}, node distance=18mm]
	 \pgfsetarrows{latex-latex};
	 \begin{scope}
		\node[rv] (0) at (0,1) {$Z$};
	 	\node[rv] (1) at (-1.25,0) {$X$};
	 	\node[rv] (2) at (0,-1) {$W$};
	 	\node[rv] (3) at (1.25,0.6) {$T$};
            \node[rv] (4) at (1.25,-0.6) {$C$};
            \begin{scope}[on background layer]
              \node[draw, gray, fill=gray!5, fit=(3)(4), inner sep=-0.75pt, rounded corners] (unobs) {};
            \end{scope}
            \node[rv] (5) at (2.25,0) {$(M, \delta)$};

	 	\draw[->] (1) -- (2);
		\draw[->] (0) -- (3);
            \draw[->] (0) -- (4);
	 	\path[->] (1) edge[bend left = 0] (3);
            \path[->] (1) edge[bend left = 0] (4);
		\path[<->] (1) edge[bend left = 30, dashed] (0);
	 	\draw[->] (2) -- (3);
            \draw[->] (2) -- (4);
		\draw[->] (0) -- (2);

        \path[<->] (3) edge[bend left = 20, dashed] (4);

            \draw[->] (3) -- (5);
            \draw[->] (4) -- (5);
	 \end{scope}
	 \end{tikzpicture}}
    \caption{Informative Case.}
    \label{fig:sfm-ic}
  \end{subfigure}
  \caption{Standard Fairness Models for different settings.}
  \label{fig:sfms}
\end{figure*}
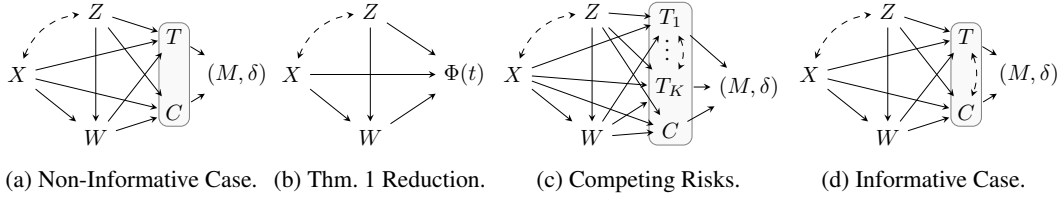

Our goal is to explain any observed discrepancies between the survival curves for the demographic groups $X = x_0$ (e.g., male) and $X = x_1$ (e.g., female). In particular, we are interested in explaining survival differences, measured by the so-called total variation measure $\text{TV}_{x_0, x_1}(t)$:
\begin{align} \label{eq:surv-tv}
         P(T > t \mid X = x_1) - P(T > t \mid X = x_0)
\end{align}
across time points $t \in [0, \infty)$. The survival curve $P(T > t)$ will be denoted by $S(t)$, and the conditional by $S(t \mid x)$.
We aim to decompose the measure of disparity in \cref{eq:surv-tv} into the contributions of (a) the direct effect, along the arrow $X \to T$; (b) indirect effect $X \to W \to T$; and (c) spurious/confounded effect $X \bidir Z \to T$. 
For decomposing the measures, we will follow a two-step process, which is closely related to the two types of assumptions encoded in the SFM in Fig.~\ref{fig:sfm-nic}. 
The first step is related to inference of $P(T \mid X, Z, W)$, whereas the second step is related to inference of causal effects of $X$ on functionals of $P(T \mid X, Z, W)$, once the conditional distribution of the survival time $P(T \mid X, Z, W)$ has been established.

\paragraph{Step 1: Inferring $P(T \mid X, Z, W)$} 
The first challenge is related to the inference of the distribution $P(T \mid X, Z, W)$ only from data on $(X, Z, W, M, \delta)$. As mentioned earlier, the core challenge stems from the fact that variables $T, C$ remain unobserved (the variables in shaded area in Fig.~\ref{fig:sfm-nic}. Instead, we have access to the pair $(M, \delta)$, and depending on the censoring mechanism, inference of $P(T \mid X, Z, W)$
may or may not be possible. 
Under the assumption $T \ci C \mid X, Z, W$ encoded in \cref{fig:sfm-nic}, however, the likelihood for $T$ is a function of $(M, \delta)$ alone, and conditional Kaplan-Meier (for survival) and Nelson-Aalen (for cumulative hazard) are valid. 

\paragraph{Step 2: Inference of Causal Effects on $\phi$.}
We next assume $P(T \mid X, Z, W)$ is known, and define $\Phi \overset{\Delta}{=} \phi(P(T \mid X, Z, W))$, where $\Phi$ depends on the random values of $X, Z, W$, whereas $\phi$ is a mapping from the space of conditional distributions of $T$ into $\mathbbm{R}$. \cref{eq:surv-tv} can thus be written as
\begin{align} \label{eq:phi-diff}
    \ex[\Phi \mid X&=x_1] - \ex[\Phi \mid X=x_0],
\end{align}
where $\Phi(t) = P(T > t \mid X, Z, W)$. 
This observation, that $\Phi(t)$ is obtained by applying a known transformation $\phi$ to an identifiable distribution $P(T \mid X, Z, W)$, allows us to construct another causal model, as described in the following result (all proofs are in \cref{appendix:proofs}):
\begin{theorem}[Reduced Standard Fairness Model]
    \label{thm:causal-reduction}
    Let $\mathcal{M}$ be an SCM compatible with the SFM in Fig.~\ref{fig:sfm-ic}. Let $\phi$ be a functional mapping $P(T \mid X, Z, W)$ to $\mathbbm{R}$. The structural mechanism of the random variable $\Phi$ is then given by:
    \begin{align}
        f_{\Phi}(x, z, w) = \phi(P(T \mid X=x, Z= z, W= w)). 
    \end{align}
    Therefore, the variable $\Phi$ is a deterministic function of $X, Z, W$, not depending on any of the noise variables $U$, and we can add $\Phi$ to the causal diagram as in \cref{fig:sfm-reduce}.
\end{theorem}
The above theorem is quite helpful as it allows us to establish the causal diagram for the new variable $\Phi(t)$. 
Then, we consider the notions of direct, indirect, and spurious effects: $x\text{-DE}_{x_0, x_1}(\Phi\mid x) = \ex[\Phi_{x_1, W_{x_0}} \mid x] - \ex[\Phi_{x_0}\mid x]$, $x\text{-IE}_{x_1, x_0}(\Phi\mid x) = \ex[\Phi_{x_1, W_{x_0}}\mid x] - \ex[\Phi_{x_1} \mid x]$, and $x\text{-SE}_{x_1, x_0}(\Phi) = \ex[\Phi_{x_1} \mid x_0] - \ex[\Phi_{x_1} \mid x_1]$. Based on these, the TV$_{x_0, x_1}(\Phi)$ measure $\ex[\Phi \mid X = x_1] - \ex[\Phi \mid X=x_0]$ can be decomposed as \citep{zhang2018fairness, plecko2022causal}:
    \begin{align} \label{eq:tv-decomp}
        \text{TV}_{x_0, x_1}(\Phi) &=  x\text{-DE}_{x_0, x_1}(\Phi\mid x_0) 
        - {x\text{-IE}_{x_1, x_0}(\Phi\mid x_0)}
        -x\text{-SE}_{x_1, x_0}(\Phi).
    \end{align}
The $x$-DE measures the contribution of an $x_0 \to x_1$ transition along the direct $X \to \Phi$ path; the $x$-IE and $x$-SE measure the reverse $x_1 \to x_0$ transition along the indirect $X \to W \to \Phi$ and spurious $X \bidir Z \to \Phi$ paths, respectively \citep{plecko2022causal}. From these, \cref{eq:tv-decomp} decomposes the marginal $\Phi$-disparity into causal pathway contributions, while \cref{thm:causal-reduction} allows us to do this at each time point. Furthermore, under the SFM in \cref{fig:sfm-reduce}, these effects are identifiable from observational data:
\begin{proposition}[Identification of Causal Effects] \label{prop:id}
    Let $\mathcal{M}$ be an SCM compatible with the Standard Fairness Model, and let $P(V)$ be its observational distribution. The potential outcome $\ex [\Phi_{x_y, W_{x_w}} (t)\mid X = x_z]$ can be identified as:
    \begin{align}
         \sum_{z,w} \ex[\Phi(t) \mid x_y, z, w]&P(W = w \mid X = x_w, Z = z) P(Z = z \mid X = x_z).
    \end{align}
\end{proposition}
\cref{prop:id} gives us a way to identify a generic potential outcome $\ex [\Phi_{x_y, W_{x_w}} (t)\mid X = x_z]$ of the survival function in the counterfactual world where $X=x_y$ along the direct, $X=x_w$ along the indirect, and $X = x_z$ along the spurious path. For estimating the causal effects on RHS of \cref{eq:tv-decomp}, different choices of $x_y, x_w, x_z \in \{0, 1\}$ need to be used as appropriate. 
The above identification results allows us to identify key causal fairness quantities in a TTE setting with non-informative censoring. We next discuss how to estimate the target quantities in practice using either (1) model-based estimation or (2) doubly-robust estimation.

\paragraph{Model-Based Estimation.}
Given our interest in mixed data types with varying dimension, and non-parametric inference, we use random survival forests (RSF) \citep{ishwaran2008random} for estimating $P(T \mid X, Z, W)$. RSF estimates the conditional cumulative hazard function $\Lambda(t \mid X, Z, W)$, by averaging the Nelson-Aalen estimator of the CHF across the leaf nodes of different trees. From $\Lambda(t \mid X, Z, W)$, the survival function is easily obtained via $S(t \mid X, Z, W) = \exp{(-\Lambda(t \mid X, Z, W))}$, allowing effect estimation:
\begin{proposition}[Model-Based Estimation of Causal Effects] \label{prop:model-est}
    Denote by $f(x, z, w)$ the estimator of $\ex[\Phi \mid x, z, w]$, and by $\hat P(x \mid v')$ the estimator of the probability $P(x \mid v')$ for different choices of $v'$. The potential outcome $\ex [\Phi_{x_y, W_{x_w}} (t)\mid X = x_z]$ can be estimated as:
    \begin{align} \label{eq:x-de-est}
        \frac{1}{n} \sum_{i=1}^n f(x_y, z_i, w_i) \frac{\hat P(x_w \mid z_i, w_i)}{\hat P(x_w \mid z_i)} \frac{\hat P(x_z \mid z_i)}{\hat P(x_z)}.
    \end{align}
\end{proposition}
\cref{prop:model-est} provides us with a basis for model-based estimation of causal effects (as opposed to doubly-robust estimation discussed in the sequel).
As mentioned earlier, the $\Phi$ variable can be thought of as time-dependent. In our setting, where $X,Z,W$ that do not vary over time, the estimated propensities $\hat P(x \mid v')$ also do not depend on time, implying that the time component of the causal decomposition is obtained in a computationally inexpensive way.

\paragraph{Doubly-Robust Estimation.} Our second approach is based on doubly-estimation of causal effects \citep{bang2005doubly}. We prove a new technical result for doubly-robust estimation of counterfactual effects based on censored data, by using the appropriate influence function for debiasing the model-based estimator:
\begin{theorem}[$\Phi(t)$ Potential Outcome Influence Function] \label{thm:dr-est}
    Let $\psi$ denote the potential outcome $\ex [\Phi_{x_y, W_{x_w}} (t)\mid X = x_z]$. Then, the influence function for $\psi$, written $\ifc{(\psi)}$, is given by:
    \begin{align}
        &\frac{\mathbbm{1}(X=x_y)\lambda_{x_z, x_w}(Z)}{P(x_z)\lambda_{x_y, x_w }( Z, W)} \left[\frac{\mathbbm{1}(M > t)}{G(t \mid X, Z, W)} + \xi_1(t) - \xi_2(t) - S(x_y, Z, W)\right] \nonumber \\
        &+ \frac{\mathbbm{1}(X=x_w)}{P(x_z)\lambda_{x_w, x_z}(Z)} \big[S(x_y, Z, W) - \nu_{x_y, x_w}(Z) \big] + \frac{\mathbbm{1}(X=x_z)}{P(x_z)} \big[ \nu_{x_y, x_w}(Z) - \psi \big],
    \end{align}
    where $G(t \mid X, Z, W) = P(C > t \mid X, Z, W)$, $S(x, Z, W) = S(t \mid x, Z, W)$, $\lambda_{x_a, x_b}(V) = \frac{P(x_a \mid V)}{P(x_b \mid V)}$, $\nu_{x_a, x_b}(Z) = \ex [S(x_a, Z, W) \mid x_b, Z]$, $\xi_1(t) = \frac{\mathbbm{1}(M \le t, \delta=0)S(t \mid X,Z,W)}{S(M \mid X,Z,W)G(M \mid X,Z,W)}$, and $\xi_2(t) = S(t \mid X,Z,W) \int_0^t \frac{\mathbbm{1}(M \ge u) h_C(u \mid X,Z,W)}{S(u \mid X,Z,W)G(u \mid X,Z,W)} du$, where $h_C$ is the conditional censoring hazard.
\end{theorem}
The proof of the theorem, together with a further discussion on the details of doubly robust estimation, is provided in \cref{appendix:dr}. We estimate $\psi$ using cross-fitting \citep{chernozhukov2018double}: nuisance functions $G, S$ and propensities are fitted on one fold (via xgboost \citep{chen2016xgboost}), while the IF is evaluated on the remaining data, and folds are swapped in turn to obtain the final one step debiased estimate.

\subsection{Causal Fairness under Competing Risks}
\label{sec:cr}
While the Cond-NIC setting considered so far is the most commonly studied scenario in TTE analysis \citep{klein1997survival}, a wide range of real-world applications require moving beyond this
basic case. In particular, individuals may often be simultaneously at risk of multiple, distinct event types, all of which are of interest to us. This leads to the setting of \emph{competing risks} (CR).

Formally, in the CR setting (see \cref{fig:sfm-cr}), each individual is at risk of $K$ event types, indexed by $J \in \{1,\dots,K\}$. Let $T_k$ denote the unobserved time to event of type $k$, and let $C$ denote the time to censoring. Available to us are only
\begin{align}
    M = \min\{T_1,\dots,T_K,C\}, \quad 
\delta \in \{0,1,\dots,K\},
\end{align}
where $\delta = k$ indicates that event $k$ occurred at time $M$, and
$\delta=0$ indicates censoring. As in \cref{sec:classic}, censoring is assumed to be conditionally independent of the event process, $C \ci (T_1,\dots,T_K)\mid X,Z,W.$
A key feature of the competing risks setting is that all event types are of interest, rather than some being treated as nuisances. While it is conceptually appealing to model the full joint distribution
$P(T_1,\dots,T_K \mid X,Z,W)$, this is generally infeasible from observed data, since only the minimum event time is observed, and the competing events may have unobserved common causes (bidirected edges among $T_1, \dots, T_K$ in \cref{fig:sfm-cr}). Therefore, identification of the joint distribution would require strong additional assumptions, such as independence
of failure times, which rarely holds. As a result, most approaches to competing risks focus on functionals of the observable process, most commonly cause-specific hazards \citep{fine1999proportional} or cumulative incidence functions (CIFs) \citep{klein1997survival}. In this work, we focus on the latter.
For cause $k$, the cumulative incidence function is defined as
\begin{align} \label{eq:cif-def}
\mathrm{CIF}_k(t \mid X,Z,W) = P(T_k \le t, T_k \leq \min_j T_j \mid X,Z,W).
\end{align}
The CIF is simply the proportion of individuals who experience event $k$ by time $t$. Importantly, CIFs explicitly account for competition between event types: increasing the rate of one event type necessarily reduces the CIFs of the remaining events. Indeed, letting
\begin{align}
    S_{\mathrm{all}}(t) = P(\min\{T_1,\dots,T_K\} > t)
\end{align}
denote the all-cause (event-free) survival function, we have
    $\sum_{k=1}^K \mathrm{CIF}_k(t) = 1 - S_{\mathrm{all}}(t)$.
This behavior differs from quantities such as $P(T_k > t \mid X,Z,W)$, which do not reflect the presence of competing events. 
These nuances need to be taken into account when interpreting causal fairness decompositions in the CR setting, where CIFs are treated as the outcome of interest. Specifically, for each cause $k$ and time point $t$, we define
\begin{align}
    \Phi_k(t) \overset{\Delta}{=} \mathrm{CIF}_k(t).
\end{align}
By \cref{thm:causal-reduction}, each $\Phi_k(t)$ is a deterministic functional of $P(T_1,\dots,T_K \mid X,Z,W)$, and hence a deterministic function of $(X,Z,W)$.
This again induces a reduced causal model in which $X$, $Z$, and $W$ influence $\Phi_k(t)$ (\cref{fig:sfm-reduce}).
Our analysis proceeds by quantifying, for each event type $k$, the
$x$-specific direct, indirect, and spurious effects of $X$ on
$\mathrm{CIF}_k(t)$. Intuitively, these effects describe how changes in the protected attribute $X$ along different causal pathways influence the cumulative incidence of event $k$ over time. In addition, we also quantify the corresponding effects on the all-cause survival function $S_{\mathrm{all}}(t)$, thereby providing both cause-specific and aggregate perspectives on survival disparities. Regarding estimation, we can again either follow a model-based or a doubly-robust approach. A discussion of estimation in the CR setting, including the derivation of influence functions for CIFs, is given in \cref{appendix:dr}.

\subsection{Causal Fairness under Informative Censoring} \label{sec:ic}
The third and final setting we consider is that of \emph{informative censoring} (IC). This setting is very common in real-world data and arises when the mechanism by which individuals are censored carries information about the event of interest. Another common setting where informative censoring occurs is when multiple terminal events are possible, but we are interested in analyzing just one of them.
Formally, in the IC setting, we consider an outcome with event time $T$ and a censoring time $C$ that is informative about $T$. Unlike the settings of NIC (\cref{sec:classic}) and CR (\cref{sec:cr}), we no longer assume that $T \ci C \mid X,Z,W$ (see \cref{fig:sfm-ic} where there is a bidirected arrow $T \bidir C$). As a result, standard tools from survival analysis are not directly applicable. 
The object of interest in this setting is typically the survival function $S(t \mid X, Z, W) =P(T > t \mid X,Z,W)$.
Specifically, in our running example, one may be interested in time-to-readmission to ICU in a world where death does not exist as a competing risk. As discussed previously, inference of such quantities is not possible without additional assumptions.
Therefore, to proceed, we adopt a sensitivity analysis approach based on assumptions about the joint law of $(T,C)$. Specifically, we assume that the joint survival function $H(t,c) = P(T > t, C > c)$
is described by an Archimedean copula
\begin{align} \label{eq:copula-coupling}
    H(t, c) = \mathcal{C}_\tau (S(t), G(c))
\end{align}
parameterized by Kendall’s $\tau$, which is our sensitivity parameter. Let $\varphi$ be the generator function of the copula, which is typically decreasing and convex. By definition, \cref{eq:copula-coupling} implies the relationship
\begin{align} \label{eq:arch-cop}
\varphi\!\left(H(t,c)\right) = \varphi\!\left(S(t)\right) + \varphi\!\left(G(c)\right),
\end{align}
where $S(t)=P(T>t)$ and $G(c)=P(C>c)$ denote the survival functions of $T$ and $C$.

\paragraph{Copula-Graphic Estimation.}
The copula-graphic estimator (CGE) exploits the relationship in
\cref{eq:arch-cop} to recover the marginal survival functions $S(t)$ and $G(c)$ from observable quantities. In the classical CGE setting, time is discretized on a grid $\{t_1,\dots,t_m\}$, and it is assumed that at each grid point only one of $S(t)$ or $G(t)$ exhibits a jump, with both functions being piecewise constant. Under these assumptions, $S(t)$ and $G(t)$ can be recursively identified from the estimable $\mathrm{CIF}_T(t), \mathrm{CIF}_C(t)$ of $T, C$. This approach is known as the CGE \citep{braekers2005copula}.
In our setting, however, estimation proceeds on a fixed grid of estimated CIFs, and we cannot assume that only one of $\hat S(t_i)$ or $\hat G(t_i)$ jumps at a given grid point. Consequently, the classical CGE must be adapted: the copula parameter $\tau$ no longer uniquely identifies $\hat S(t_{i+1})$ and $\hat G(t_{i+1})$ from the estimated CIFs, but identifies their bounds:
\begin{theorem}[CGE bounds under non-disjoint increments -- informal] \label{thm:cge-bounds}
Consider an Archimedean copula with generator $\varphi$ and parameter $\tau$. Given $(S(t_i), G(t_i))$ and the CIF increments $\Delta\mathrm{CIF}_T, \Delta\mathrm{CIF}_C$ over $(t_i, t_{i+1}]$, there exist sharp bounds $\underline{S}(t_{i+1}) \leq S(t_{i+1}) \leq \overline{S}(t_{i+1})$ and $\underline{G}(t_{i+1}) \leq G(t_{i+1}) \leq \overline{G}(t_{i+1})$, computable in closed form from $\varphi$ and $(S(t_i), G(t_i))$. The bounds collapse to point estimates as $\max_i |t_{i+1} - t_i| \to 0$.
\end{theorem}
The formal statement and proof are in \cref{appendix:copula-graphic}. After establishing bounds, we obtain point estimates via midpoint interpolation, $\hat S(t_{i+1}) = \tfrac12(\underline{S}(t_{i+1}) + \overline{S}(t_{i+1}))$, and analogously for $\hat G(t)$, yielding a point-identified estimator of $S(t)$ as the grid becomes refined.
Finally, after recovering $\hat S(t)$, which is the target quantity in the informative censoring setting, we may once again invoke
\cref{thm:causal-reduction} and compute the $x$-specific direct, indirect, and spurious effects of $X$ on $S(t)$. This allows us to quantify causal fairness under informative censoring, through a sensitivity analysis based on the $\tau$ parameter.

\paragraph{Two Modeling Routes for Informative Censoring.}
The copula assumption in \cref{eq:arch-cop} may be imposed at two distinct levels of granularity, leading to two modeling approaches.

\emph{Route I (Conditional Copula Given Covariates).} The first approach is to model the joint $P(T, C \mid X,Z,W)$ by an Archimedean copula with parameter $\tau$. In this case, the dependence structure for $(T,C)$ is specified conditionally on $(X,Z,W)$. However, in the IC setting, the survival $S(t \mid X, Z, W)$ is obtained recursively, with each step involving a complex non-linear map based on $\varphi$, making doubly-robust estimation difficult. We thus focus on model-based estimation for this route.

\emph{Route II (Population-Level Copula for Potential Outcomes).}
The second approach is to impose the copula assumption on the joint law of the \emph{potential outcomes} $(T,C)$ after marginalization over $(Z,W)$, that is, on quantities of the form
\begin{align}
    P ( T_{x_y,W_{x_w}}, C_{x_y,W_{x_w}} \mid x_z).
\end{align}
In this case, the copula allows us to reconstruct population-level survival quantities from $\mathrm{CIF}_T, \mathrm{CIF}_C$, which can be estimated via a doubly-robust procedure. 
The sensitivity parameter $\tau$ is thus interpreted as governing dependence at the population level.  The final challenge in this approach lies in the propagation of uncertainty over the CIF estimates into the uncertainty over the $P( T_{x_y,W_{x_w}} > t \mid x_z)$, which we discuss in Appx.~\ref{appendix:envelopes}.
\section{Case Study -- Race Differentials in ICU} \label{sec:experiments}
Finally, we showcase the usefulness of our methodology on a real-world study. Our task is to study racial disparities in outcome after admission to an intensive care unit (ICU). 
Disparities in post-ICU outcomes across race and ethnicity are known to exist \citep{mcgowan2022racial, plecko2025algorithmic}, although the topic has not been studied extensively in the literature. 
We make use of the data Adult Patient Database (APD) of the Australia and New Zealand Intensive Care Society (ANZICS) \citep{stow2006development, secombe2023thirty}, and the acknowledgement of contributing hospitals is given in \cref{appendix:anzics-ack}. 
The ANZICS APD receives submissions from 98\% of ICUs in Australia. 
We differentiate four variable groups: indigenous status;
sociodemographic variables (age, sex, postcode-based SES);
ICU-relevant clinical variables, including the APACHE-III
\citep{knaus1991apache} predicted risk of death, ANZICS
modified APACHE-III admission diagnosis (see
\href{https://www.anzics.org/wp-content/uploads/2021/03/ANZICS-APD-Dictionary-Version-6.1.pdf}{full list} of
considered diagnoses), and indicator of whether admission
was elective; and outcomes $\{T_1, T_2\}$, where $T_1$ is
time to death after ICU admission as registered in the
database or the National Death Index (NDI), while $T_2$ is
the time to readmission to ICU as recorded in the database.
\begin{figure}[t]
    \centering
    \includegraphics[width=\textwidth]{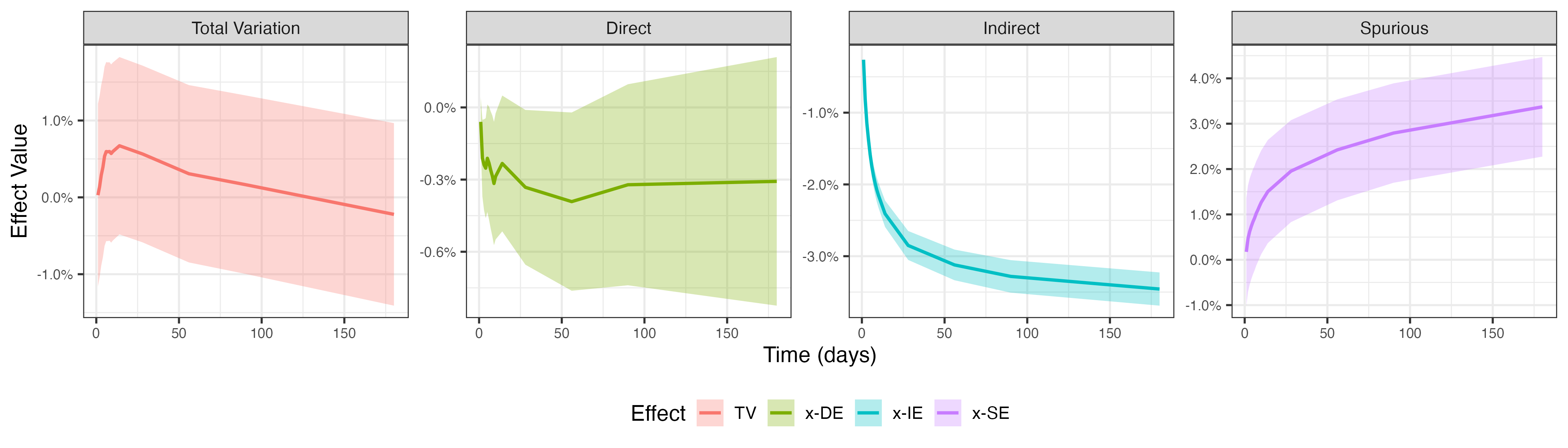}
    \caption{Case study results: survival curves, disparity metrics, and causal decompositions.}
    \label{fig:death-effects}
\end{figure}

We now justify the causal structure. Indigenous status and
sociodemographic variables may share common causes through
historical and socioeconomic processes, but neither is a
cause of the other, motivating the $X \bidir Z$
relationship and follows the modeling choice adopted in previous literature \citep{mcgowan2022racial, plecko2025algorithmic}. Clinical variables are temporally downstream
from both: for instance, indigenous status may influence
illness severity through disparities in healthcare access,
chronic disease burden, and preventive care. All variables
may influence mortality and readmission. This
justifies treating indigenous status as $X$,
sociodemographics as confounders $Z$, clinical variables
as mediators $W$, and survival times as the outcome.

We acknowledge two potential limitations regarding
unmeasured confounding. First, genetic variation may
act as a common cause of indigenous status and disease
susceptibility, confounding the $X \to W$ relationship.
Second, geographic remoteness may confound multiple
relationships: it can independently influence indigenous
status composition, clinical presentation at admission,
and post-discharge outcomes ($X$, $W$, and $T$).
Sensitivity analysis for such unmeasured confounding
in mediation settings is an important direction for
future work.
We analyzed all admissions from 181 hospitals across Australia between January 1st 2023 to July 1st 2024, yielding a cohort of $249,958$ patients, of whom $10,436$ were Indigenous, $P(X=x_0) \approx 4.2\%$.
In our setting, the censoring mechanism is related to the data extraction process (administrative reason), and hence the key assumption of non-informative censoring, implied by the conditional independence of $T, C$ given $X, Z, W$ is justified. 

\paragraph{Non-Informative Censoring -- Death Disparities.} 
We first focus on differences in death rates, studying how survival disparities evolve over time.
We consider non-zero values of direct, indirect, or spurious effects as evidence of disparities between groups.
Due to administrative censoring, our setting falls under non-informative censoring (NIC) studied in \cref{sec:classic}. 
We begin by estimating the survival differences for $x_0, x_1$ populations over time, measured by the TV measure in \cref{eq:surv-tv}, visualized in \cref{fig:death-effects} (far left), which serves as a statistical baseline commonly used in the literature \citep{rahman2022fair, liu2025equitable}. 
As the plot shows, in the first 3 weeks after ICU admission, the survival is increasingly higher for the majority group $x_1$, but after this the TV measure drops back down, with the point estimate negative at $T=180$ days, meaning that minority patients have higher survival probability at that point in time. 
However, 95\% confidence intervals (obtained assuming asymptotic normality of our cross-fitted doubly robust approach) include $0$ after $T=80$ days.

We then apply the temporal decomposition discussed in \cref{sec:classic}, breaking $P(T > t \mid X = x_1) - P(T > t \mid X = x_0)$ into contributions from the direct, indirect, and spurious pathways, for different values of $t$. 
The decomposition is shown in \cref{fig:death-effects} (three facets to the right), and highlights interesting results that may be possibly overlooked by a classical statistical approach. 
First, we see that the direct effect decreases up to $T=50$ days after ICU admission, after which it stabilizes. The sign of this effect implies that minority patients experience \emph{improved survival} along the direct path.
Secondly, we see that the indirect effect decreases over time. The sign of the effect implies that minority patients have \emph{reduced survival} along this causal pathway (due to higher burden of chronic and acute illness), which reaches almost $3\%$ difference after $T=50$ days, and continues to amplify up to $T = 180$ days. We remark these are rather large effects in a cohort with a marginal mortality rate of $< 10\%$.
Thirdly, for the spurious effect, we see an increasing value up to $T = 180$ days. This implies that minority patients have improved survival along the spurious pathway, which can be related to lower overall age at admission, which naturally reduces the mortality rate.
Taken together, the decomposition reveals that the near-zero TV measure masks substantial, opposing pathway-specific effects, uncovering a pattern that would be missed by standard statistical fairness analyses.

\begin{figure}[t]
    \centering
    \includegraphics[width=\textwidth]{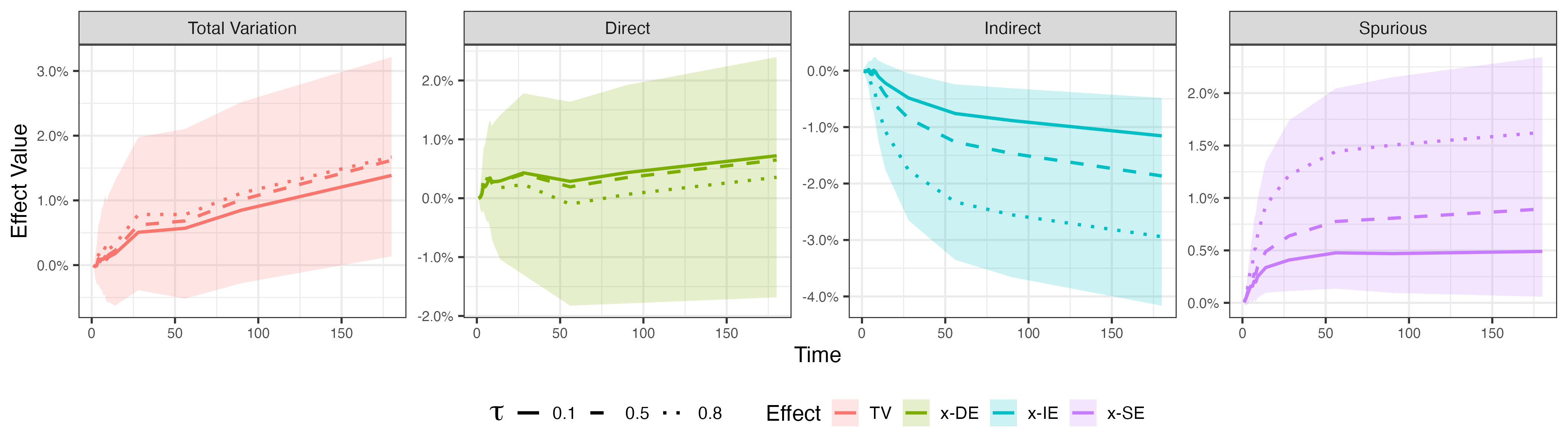}
    \caption{Case study results: survival curves, disparity metrics, and causal decompositions.}
    \label{fig:readm-effects}
\end{figure}
\paragraph{Informative Censoring -- Readmission Disparities.}
In the second analysis, we focus on the readmission outcome as our primary target. Readmission is a relevant post-ICU outcome since it has a slightly different etiology than death, as it reflects downstream clinical stability, discharge decisions, and post-ICU recovery processes rather than acute physiological failure.
Furthermore, readmission is also interesting as it has a direct implication on the burden of care placed on health systems. 
In this context, we wish to study the disparities in time-to-readmission $P(T_2 > t)$, in the hypothetical world where death does not act as a censoring outcome (clearly, in the real world, death prevents us from observing a possible need for readmission).
Due to informativeness of death censoring for readmission, which has unobserved common causes with readmission (latent illness severity, unmeasured frailty, treatment limitations, post-discharge care quality), our setting falls under informative censoring (IC).
We perform the sensitivity analysis from \cref{sec:ic}, and use Route II described therein for estimation of causal effects.
Based on domain knowledge, we assume that event times $T_1, T_2$ are positively correlated ($\tau > 0$), as they share common causes that make both events more likely, including functional reserve, complications during ICU stay, and post-discharge care.

The results of the sensitivity analysis with $\tau \in \{0.1, 0.5, 0.8\}$ are shown in \cref{fig:readm-effects}, where the ribbons indicate maximal and minimal 95\% confidence intervals across $\tau$ values. 
The TV measure increases over time, meaning that burden of readmission is higher for minority patients. Along the direct effect, minority patients require more readmissions, opposite to the finding for the death outcome. 
However, the confidence bounds for these measures do not allow us to make definitive conclusions about effect sizes.
Next, along the indirect effect, the burden of readmission is increased for the minority group, due to more complicated admission diagnoses and worse chronic health compared to the majority population. Finally, for the spurious effect, minority patients again have a lower burden of readmission, likely due to lower average age at time of ICU admission. 

\paragraph{Limitations \& Future Work.}
Our framework rests on the lack-of-confounding assumptions encoded in the SFM (Fig.~\ref{fig:sfm-nic}). While these assumptions are prevalent in the causal fairness literature \citep{plecko2022causal, zhang2018fairness}, they should not be taken for granted, but carefully justified on a case-by-case basis. 
Developing sensitivity analyses for causal decompositions in presence of unmeasured confounding is an important direction for future work. A further limitation is the reliance on the Archimedean copula assumption in the informative censoring setting (Sec.~\ref{sec:ic}); while we treat $\tau$ as a sensitivity parameter, the choice of copula family itself is an additional modeling assumption that future work could relax.
Finally, our formulation currently assumes time-invariant covariates; and our competing risks analysis does not isolate primary from competing hazard effects. We leave these extensions for future work.

\newpage

\ifnum\treportflag=0
\section*{Broader Impact Statement}
This work introduces causal fairness methodology for analyzing disparities in time-to-event settings. The proposed methods are intended as analytical tools for understanding and quantifying inequities, rather than for automated decision-making. We do not foresee direct ethical risks or potential harms arising from this work. On the contrary, by enabling more principled analyses of disparities in survival settings, we hope the methods may contribute to identifying, understanding, and ultimately reducing unfair outcomes in applied domains.
\fi

\bibliographystyle{abbrvnat}
\bibliography{refs}

\newpage
\appendix
\section*{\centering\Large Technical Appendices for \textit{Causal Fairness for Survival Analysis}}
The source code for reproducing all the experiments can be found in our anonymized code repository \url{https://anonymous.4open.science/r/cfa-survival-B880}. 
All experiments were performed on a MacBook Pro, with the M3 Pro chip and 36 GB RAM on macOS 26.2 (Tahoe). All experiments can be run with less than 24 hours of compute on the above-described machine or equivalent.

\section{Proofs} \label{appendix:proofs}

\begin{proof}[\cref{thm:causal-reduction} Proof:]
Let $\mathcal{M}$ be an SCM compatible with the SFM in \cref{fig:sfm-ic}, and let $f_T$ be the structural mechanism of $T$, taking $X, Z, W$ as inputs together with the noise variables $U_T, U_C$. The conditional distribution $P(T \mid X=x, Z=z, W=w)$ is determined by $f_T$ and the distribution of the noise variables, and depends on $X, Z, W$ only through their realized values $x, z, w$. Therefore, the random variable $\Phi \overset{\Delta}{=} \phi(P(T \mid X, Z, W))$ has the structural mechanism
\begin{align}
    f_{\Phi}(x, z, w) = \phi(P(T \mid X=x, Z=z, W=w)),
\end{align}
which is a deterministic function of $X, Z, W$ and does not depend on any of the noise variables $U$. Consequently, $\Phi$ can be added to the causal diagram as in \cref{fig:sfm-reduce}, completing the proof.
\end{proof}

\begin{proof}[\cref{prop:id} Proof:]
We demonstrate that the potential outcome $\ex[\Phi_{x_y, W_{x_w}}(t) \mid X = x_z]$ is identifiable under the Standard Fairness Model. For this purpose, we make use of the counterfactual graph \citep{shpitser2012counterfactuals} in \cref{fig:ctf-sfm}. We expand $\ex[\Phi_{x_y, W_{x_w}}(t) \mid x_z]$ as follows:
\begin{align}
    &= \sum_{w} \ex[\Phi_{x_y, w}(t) \mathbbm{1}(W_{x_w} = w) \mid x_z] \quad \text{(Counterfactual Un-nesting)} \\
    &= \sum_{z, w} \ex[\Phi_{x_y, w}(t) \mathbbm{1}(W_{x_w} = w) \mid z, x_z] P(z \mid x_z) \quad \text{(Law of Total Probability)} \\
    &= \sum_{z, w} \ex[\Phi_{x_y, w}(t) \mid z, x_z] P(W_{x_w} = w \mid z, x_z) P(z \mid x_z) \quad (\Phi_{x_y, w} \ci W_{x_w} \mid Z, X) \\
    &= \sum_{z, w} \ex[\Phi_{x_y, w}(t) \mid z, x_z] P(W_{x_w} = w \mid z, x_w) P(z \mid x_z) \quad (W_{x_w} \ci X \mid Z) \\
    &= \sum_{z, w} \ex[\Phi_{x_y, w}(t) \mid z, x_z] P(W = w \mid z, x_w) P(z \mid x_z) \quad \text{(Consistency)} \\
    &= \sum_{z, w} \ex[\Phi_{x_y, w}(t) \mid z, w, x_z] P(W = w \mid z, x_w) P(z \mid x_z) \quad (\Phi_{x_y, w} \ci W \mid Z, X) \\
    &= \sum_{z, w} \ex[\Phi_{x_y, w}(t) \mid z, w, x_y] P(W = w \mid z, x_w) P(z \mid x_z) \quad (\Phi_{x_y, w} \ci X \mid Z, W) \\
    &= \sum_{z, w} \ex[\Phi(t) \mid z, w, x_y] P(W = w \mid z, x_w) P(z \mid x_z) \quad \text{(Consistency)},
    \label{eq:phi-id}
\end{align}
which completes the proof.
\end{proof}

\begin{proof}[\cref{prop:model-est} Proof:]
We focus on the quantity $\ex[\Phi_{x_y, W_{x_w}}(t) \mid x_z]$. Suppose that we have:
\begin{align} \label{eq:consist-fxyzw}
    f(x_y, z, w) &\xrightarrow{P} \ex[\Phi(t) \mid x_y, z, w] \;\forall z, w \\
    \hat P(x_w \mid w, z) &\xrightarrow{P} P(x_w \mid w, z) \;\forall z, w \\
    \hat P(x_w \mid z) &\xrightarrow{P} P(x_w \mid z) \;\forall z \\
    \hat P(x_z \mid z) &\xrightarrow{P} P(x_z \mid z) \;\forall z \\
    \hat P(z, w) &\xrightarrow{P} P(z, w) \;\forall w, z \\
    \hat P(x_z) &\xrightarrow{P} P(x_z) \label{eq:consist-pxz}
\end{align}
where $f$ and the $\hat P(\cdot \mid \cdot)$ terms are estimators, and $\hat P(z, w), \hat P(x_z)$ are the empirical distributions:
\begin{align}
    \hat P(z, w) &= \frac{1}{n} \sum_{i=1}^n \mathbbm{1}(Z_i = z, W_i = w) \\
    \hat P(x_z) &= \frac{1}{n} \sum_{i=1}^n \mathbbm{1}(X_i = x_z).
\end{align}
Applying Bayes' rule to the identification expression in Eq.~\ref{eq:phi-id}, we re-write it as:
\begin{align}
    \sum_{z, w} \ex[\Phi(t) \mid z, w, x_y] \, P(z, w) \, \frac{P(x_w \mid w, z)}{P(x_w \mid z)} \, \frac{P(x_z \mid z)}{P(x_z)}.
\end{align}
By a coupling of quantities in Eqs.~\ref{eq:consist-fxyzw}-\ref{eq:consist-pxz} to a joint probability space and an application of the continuous mapping theorem we obtain that
\begin{align}
    &\sum_{z, w} f(x_y, z, w) \, \hat P(z, w) \, \frac{\hat P(x_w \mid w, z)}{\hat P(x_w \mid z)} \, \frac{\hat P(x_z \mid z)}{\hat P(x_z)} \\
    &\xrightarrow{P} \sum_{z, w} \ex[\Phi(t) \mid z, w, x_y] \, P(z, w) \, \frac{P(x_w \mid w, z)}{P(x_w \mid z)} \, \frac{P(x_z \mid z)}{P(x_z)}.
\end{align}
Finally, note that we have
\begin{align}
    &\sum_{z, w} f(x_y, z, w) \hat P(z, w) \frac{\hat P(x_w \mid w, z)}{\hat P(x_w \mid z)} \frac{\hat P(x_z \mid z)}{\hat P(x_z)} \\
    &= \sum_{z, w} f(x_y, z, w) \Big(\sum_{i=1}^n \frac{\mathbbm{1}(Z_i = z, W_i = w)}{n}\Big) \frac{\hat P(x_w \mid w, z)}{\hat P(x_w \mid z)} \frac{\hat P(x_z \mid z)}{\hat P(x_z)} \\
    &= \frac{1}{n} \sum_{i=1}^n f(x_y, z_i, w_i) \frac{\hat P(x_w \mid w_i, z_i)}{\hat P(x_w \mid z_i)} \frac{\hat P(x_z \mid z_i)}{\hat P(x_z)},
\end{align}
completing the proof of the proposition.
\end{proof}
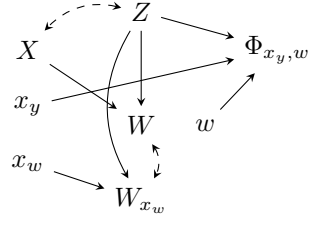
\begin{figure}[t]
    \centering
    \begin{tikzpicture}
	 [>=stealth, rv/.style={thick}, rvc/.style={triangle, draw, thick, minimum size=7mm}, node distance=18mm]
	 \pgfsetarrows{latex-latex};
	 \begin{scope}
		\node[rv] (Z) at (0,1) {$Z$};
	 	\node[rv] (X) at (-1.5,0.5) {$X$};
	 	\node[rv] (W) at (0,-0.5) {$W$};
            \node[rv] (Phi) at (1.8,0.5) {$\Phi_{x_y, w}$};
            \node[rv] (Wxw) at (0,-1.5) {$W_{x_w}$};
            \node[rv] (xy) at (-1.5,-0.25) {${x_y}$};
            \node[rv] (xw) at (-1.5,-1) {${x_w}$};
            \node[rv] (w) at (0.85,-0.5) {${w}$};
            \draw[->] (Z) -- (W);
            \draw[->] (X) -- (W);
            \draw[->] (xy) -- (Phi);
            \draw[->] (w) -- (Phi);
            \draw[->] (xw) -- (Wxw);
            \draw[->] (Z) -- (Phi);
            \draw[->] (Z) edge[bend left = -30] (Wxw);
		\path[<->] (Z) edge[bend left = -30, dashed] (X);
            \path[<->] (W) edge[bend left = 30, dashed] (Wxw);
	 \end{scope}
    \end{tikzpicture}
    \caption{Counterfactual graph of the SFM used in the proof of \cref{prop:id}.}
    \label{fig:ctf-sfm}
\end{figure}
The proof of \cref{thm:dr-est} is given in \cref{appendix:dr}, while the proof of \cref{thm:cge-bounds} is given in \cref{appendix:copula-graphic}.
\newpage
\section{Alternative Scales for Quantifying Survival Disparities}
\label{appendix:alt-scales}
In this appendix, we discuss alternative scales for quantifying survival disparities across demographic groups, beyond the survival difference metric used in the main text (see \cref{eq:surv-tv}). The alternative metrics that may be of interest include restricted mean survival time (RMST) differences \citep{royston2013restricted} and (cumulative) hazard ratios \citep{klein1997survival}. Importantly, all such metrics can be accommodated within the causal framework developed in this paper.
Formally, we may be interested in a metric $d\big(P(T \mid X = x_1),\, P(T \mid X = x_0)\big)$, where $P(T \mid X = x)$ denotes the survival distribution for group $X=x$. The possible quantities of interest include:
\begin{enumerate}[label=(\roman*)]
    \item \label{metric:surv} \emph{Survival differences} (used in \cref{sec:classic})
    \begin{align}
        P(T > t \mid X = x_1) - P(T > t \mid X = x_0),
        \qquad t \in [0, \infty),
    \end{align}
    \item \label{metric:rmst} \emph{Restricted mean survival time (RMST) differences}
    \begin{align}
        \ex\!\left[ \min\{T, t\} \mid X = x_1 \right]
        -
        \ex\!\left[ \min\{T, t\} \mid X = x_0 \right],
        \qquad t \in [0, \infty),
    \end{align}
    \item \label{metric:hazard} \emph{(Cumulative) hazard ratios}
    \begin{align}
        \frac{\ex\!\left[\eta(t \mid X, Z, W) \mid X = x_1\right]}
             {\ex\!\left[\eta(t \mid X, Z, W) \mid X = x_0\right]},
        \qquad t \in [0, \infty),
    \end{align}
    where $\eta(t \mid X, Z, W)$ denotes the (cumulative) hazard function.
\end{enumerate}
These metrics provide different scales on which time-evolving disparities between groups may be quantified (depending on the subject question of interest). Regardless of the chosen metric, our objective is to decompose each such disparity measure into contributions arising from the direct path $X \to T$, the indirect path $X \to W \to T$, and the spurious path $X \bidir Z \to T$.

\paragraph{Reduction to a Deterministic Outcome.}
As described in \cref{sec:classic}, the key observation enabling a unified treatment of the above disparity measures is that each can be written as a functional of the conditional distribution $P(T \mid X,Z,W)$. Specifically, let $\phi$ be a functional mapping $P(T \mid X,Z,W)$ to $\mathbbm{R}$, and define the random variable
\begin{align}
    \Phi \;\overset{\Delta}{=}\; \phi\!\left(P(T \mid X,Z,W)\right).
\end{align}
Here, $\Phi$ is a deterministic function of $(X,Z,W)$, while $\phi$ operates on the conditional survival distribution. With this notation, each metric above can be written either on a difference scale,
\begin{align}\label{eq:phi-diff-2}
    \ex[\Phi \mid X = x_1] - \ex[\Phi \mid X = x_0],
\end{align}
or on a ratio scale,
\begin{align}
    \frac{\ex[\Phi \mid X = x_1]}{\ex[\Phi \mid X = x_0]}.
    \label{eq:phi-ratio}
\end{align}
For example, taking $\Phi = P(T>t \mid X,Z,W)$ in \cref{eq:phi-diff-2} yields the survival difference metric in \ref{metric:surv}, while taking $\Phi = \eta(t \mid X,Z,W)$ in \cref{eq:phi-ratio} yields the cumulative hazard ratio metric in \ref{metric:hazard}.
As described in \cref{thm:causal-reduction}, since $\Phi$ is a deterministic function of $(X,Z,W)$, it does not depend on any exogenous noise variables. Consequently, $\Phi$ may be added as a node to the causal graph implied by the Standard Fairness Model (SFM), inheriting the same causal structure as any other deterministic transformation of $(X,Z,W)$.
This observation allows us to define direct, indirect, and spurious effects of $X$ on $\Phi$, and the associated TV decomposition as in \cref{eq:tv-decomp}. The RMST metric in \ref{metric:rmst} can be analyzed using the same decomposition. Ratio metrics as in \ref{metric:hazard}, however, require a different approach. Here we provide the formal result that allows us to decompose statistical differences on the ratio scale into their causal contributions along pathways:
\begin{proposition}[$x$-specific Decomposition for Ratio Scales] \label{prop:x-decomp-ratio}
    Let $\Phi$ denote a random variable. The $x$-specific direct, indirect, and spurious ratio-scale effects of $X$ on $\Phi$ are defined as
    \begin{align}
        x\text{-DR}_{x_0,x_1}(\Phi\mid x)
        &= \frac{\ex[\Phi_{x_1,W_{x_0}} \mid X = x]}{\ex[\Phi_{x_0,W_{x_0}} \mid X = x]}, \\
        x\text{-IR}_{x_1,x_0}(\Phi\mid x)
        &= \frac{\ex[\Phi_{x_1,W_{x_0}} \mid X = x]}{\ex[\Phi_{x_1} \mid X = x]}, \\
        x\text{-SR}_{x_1,x_0}(\Phi)
        &= \frac{\ex[\Phi_{x_1} \mid X = x_0]}{\ex[\Phi_{x_1} \mid X = x_1]}.
    \end{align}
    Using these measures, the total variation ratio $\frac{\ex[\Phi \mid x_1]}{\ex[\Phi \mid x_0]}$ can be decomposed as:
    \begin{align}
        \frac{\ex[\Phi \mid x_1]}{\ex[\Phi \mid x_0]}
    &=
    x\text{-DR}_{x_0,x_1}(\Phi\mid x_0)
    \times
    \big[x\text{-IR}_{x_1,x_0}(\Phi\mid x_0)\big]^{-1}
    \times
    \big[x\text{-SR}_{x_1,x_0}(\Phi)\big]^{-1}.
    \end{align}
\end{proposition}
\cref{prop:x-decomp-ratio} can be seen as an analogue of the result in \cref{eq:tv-decomp} that is useful for metrics where ratios are more natural than differences, such as hazard ratios or cumulative hazard ratios \citep{klein1997survival}.
\newpage
\section{Doubly Robust Estimation}
\label{appendix:dr}
In this appendix, we describe the construction of influence functions for the causal estimands used in \cref{sec:classic}. We first focus on the potential outcomes of the survival curve $S(t \mid X, Z, W)$. For a fixed $t \geq 0$, our target parameter $\psi(t)$ is the potential outcome $\ex\!\left[ S_{x_y, W_{x_w}}(t) \mid X = x_z \right]$, where $S_{x_y, W_{x_w}}(t) = P(T_{x_y, W_{x_w}} > t)$.  
Under the Standard Fairness Model (Fig.~\ref{fig:sfm-nic}) with non-informative censoring $T \ci C \mid X,Z,W$, \cref{prop:id} shows that the identification expression for $\psi (t)$ is given by
\begin{align}\label{eq:psi-id}
    \psi(t) =\sum_{z,w}S(t \mid x_y, z, w)\,P(w \mid x_w, z)\,P(z \mid x_z).
\end{align}
As before, we denote the observed data quantities $M=\min\{T,C\}$ and $\delta=\mathbbm{1}(T\le C)$. Due to independence, note that
\begin{align}
    S(t \mid X,Z,W)
    =
    \ex\!\left[
        \frac{\mathbbm{1}(M>t)}{G(t\mid X,Z,W)}
        \,\middle|\,
        X,Z,W
    \right],
\end{align}
where $G(t\mid X,Z,W)=P(C>t\mid X,Z,W)$ is the censoring survival function. 

\paragraph{Influence Function Derivation.}
We then compute the influence function of the identification expression in \cref{eq:psi-id}, which can be obtained as:
\begin{align}
    \ifc(\psi(t))
    &=
    \sum_{z,w}
    \underbrace{
        \ifc\!\left( S(t \mid x_y, z, w) \right)
        P(w \mid x_w, z) P(z \mid x_z)
    }_{T_1}
    \\
    &\quad+
    \underbrace{
        S(t \mid x_y, z, w)\,
        \ifc\!\left(P(w \mid x_w, z)\right)
        P(z \mid x_z)
    }_{T_2}
    \nonumber \\
    &\quad+
    \underbrace{
        S(t \mid x_y, z, w)\,
        P(w \mid x_w, z)\,
        \ifc\!\left(P(z \mid x_z)\right)
    }_{T_3},
\end{align}
using the fact that $\ifc (A B) = \ifc(A)\cdot B + A \cdot \ifc (B)$. Let $\xi_1(t)$ and $\xi_2(t)$ denote the augmentation terms defined in \cref{thm:dr-est}. The individual influence functions for each of the appearing terms can be obtained as:
\begin{align}
    \ifc\!\left(S(t \mid x_y, z, w)\right)
    &=
    \frac{\mathbbm{1}(X=x_y,Z=z,W=w)}{P(x_y,z,w)} \\
        &\quad \cdot \left[
        \frac{\mathbbm{1}(M>t)}{G(t\mid X,Z,W)}
        {+ \xi_1(t) - \xi_2(t)}
        -
        S(t\mid x_y,z,w)
    \right],
    \label{eq:if-surv-t1}
    \\
    \ifc\!\left(P(w \mid x_w, z)\right)
    &=
    \frac{\mathbbm{1}(X=x_w,Z=z)}{P(x_w,z)}
    \left[
        \mathbbm{1}(W=w)-P(w\mid x_w,z)
    \right],
    \\
    \ifc\!\left(P(z \mid x_z)\right)
    &=
    \frac{\mathbbm{1}(X=x_z)}{P(x_z)}
    \left[
        \mathbbm{1}(Z=z)-P(z\mid x_z)
    \right].
\end{align}

Substituting into $T_1$ gives
\begin{align} \label{eq:st-if}
    \hspace{-4pt}\frac{\mathbbm{1}(X=x_y)}{P(x_z)}
    \frac{P(x_z\mid Z)}{P(x_w\mid Z)}
    \frac{P(x_w\mid Z,W)}{P(x_y\mid Z,W)} 
    \left[
        \frac{\mathbbm{1}(M>t)}{G(t\mid X,Z,W)}
        {+ \xi_1(t) - \xi_2(t)}
        -
        S(t\mid X,Z,W)
    \right],
\end{align}
whereas substituting into $T_2, T_3$ yields
\begin{align}
    T_2
    &=
    \frac{\mathbbm{1}(X=x_w)}{P(x_z)}
    \frac{P(x_z\mid Z)}{P(x_w\mid Z)}
    \left[
        S(t\mid x_y,Z,W)
        -
        \ex[S(t\mid x_y,Z,W) \mid x_w,Z]
    \right],
    \\ \label{eq:term-3}
    T_3
    &=
    \frac{\mathbbm{1}(X=x_z)}{P(x_z)}
    \ex[S(t\mid x_y,Z,W) \mid x_w,Z]
    -
    \frac{\mathbbm{1}(X=x_z)}{P(x_z)}\,\psi(t).
\end{align}
Combining the terms yields the influence function stated in \cref{thm:dr-est}.

\paragraph{Double Robustness.}
Let $S, G$ denote the true survival and censoring functions, and let $\hat{S}, \hat{G}$ denote their estimates. Let the true censoring density be $f_C(u) = -G'(u)$ and the estimated censoring hazard be $\hat{h}_C(u) = \frac{-\hat{G}'(u)}{\hat{G}(u)}$. Because $T \ci C \mid X,Z,W$, the joint density of $(T,C)$ factors as $f_T(t)f_C(c)$ conditionally. For brevity, we suppress covariate conditioning. 
The term in the bracket in Eq.~\ref{eq:st-if} is $D(t) = \frac{\mathbbm{1}(M>t)}{\hat{G}(t)} + \xi_1(t) - \xi_2(t) - S(t)$, where:
\begin{align}
    \xi_1(t) = \hat{S}(t)\frac{\mathbbm{1}(M \le t, \delta=0)}{\hat{S}(M)\hat{G}(M)}, \quad \xi_2(t) = \hat{S}(t)\int_0^t \frac{\mathbbm{1}(M \ge u)}{\hat{S}(u)\hat{G}(u)} \hat{h}_C(u) du.
\end{align}
We want to show consistency $\ex[D(t)] = 0$ under misspecification of either $S$ or $G$. 

\emph{Case 1 ($\hat{G} = G$, $\hat{S}$ is arbitrary).} Because the estimated censoring model matches the true distribution, we have $\hat{G}(u) = G(u)$ and the estimated hazard is the true hazard, $\hat{h}_C(u) = \frac{-G'(u)}{G(u)}$. The expected value of the IPCW term is simply $\ex[ \frac{\mathbbm{1}(M>t)}{{G}(t)} ] = \frac{P(T>t, C>t)}{G(t)} = \frac{S(t)G(t)}{G(t)} = S(t)$.
The censoring event $\delta = 0$ in term $\xi_1(t)$ triggers only when $C \le t$ and $T > C$, so the term can be re-written as an integral:
\begin{align}
    \ex[\xi_1(t)] &= \hat{S}(t) \int_0^t \int_c^\infty \frac{1}{\hat{S}(c) G(c)} f_T(u) f_C(c) \,du \,dc \nonumber \\
    &= \hat{S}(t) \int_0^t \frac{P(T > c)}{\hat{S}(c) G(c)} f_C(c) \,dc = \hat{S}(t) \int_0^t \frac{S(c)}{\hat{S}(c) G(c)} (-G'(c)) \,dc. \label{eq:xi1-case1}
\end{align}
For the term $\xi_2(t)$, we have $\ex[\mathbbm{1}(M \ge u)] = P(T \ge u, C \ge u) = S(u)G(u)$. Thus:
\begin{align}
    \ex[\xi_2(t)] &= \hat{S}(t) \int_0^t \frac{\ex[\mathbbm{1}(M \ge u)]}{\hat{S}(u) G(u)} \left( \frac{-G'(u)}{G(u)} \right) du = \hat{S}(t) \int_0^t \frac{S(u)}{\hat{S}(u) G(u)} (-G'(u)) \,du. \label{eq:xi2-case1}
\end{align}
Comparing \cref{eq:xi1-case1} and \cref{eq:xi2-case1}, we see $\ex[\xi_1(t)] = \ex[\xi_2(t)]$, so their difference strictly evaluates to zero. Hence, $\ex[D(t)] = \ex[\frac{\mathbbm{1}(M>t)}{\hat{G}(t)}] - S(t) = 0$.

\emph{Case 2 ($\hat{S} = S$, $\hat{G}$ is arbitrary).} Here, the IPCW exibits a bias, namely $\ex[ \frac{\mathbbm{1}(M>t)}{\hat{G}(t)} ] = \frac{S(t)G(t)}{\hat{G}(t)} \neq S(t)$. 
Similarly to \cref{eq:xi1-case1,eq:xi2-case1}, we evaluate the expected augmentation terms, this time under the correct $S$ but misspecified $\hat{G}$:
\begin{align}
    \ex[\xi_1(t)] &= S(t) \int_0^t \frac{-G'(c)}{\hat{G}(c)} \,dc, \label{eq:xi1-case2} \\
    \ex[\xi_2(t)] &= S(t) \int_0^t \frac{S(u)G(u)}{S(u)\hat{G}(u)} \left( \frac{-\hat{G}'(u)}{\hat{G}(u)} \right) du = S(t) \int_0^t \frac{-G(u)\hat{G}'(u)}{\hat{G}(u)^2} \,du. \label{eq:xi2-case2}
\end{align}
Subtracting \cref{eq:xi2-case2} from \cref{eq:xi1-case2} results in an integral of a differentiated quotient $\frac{-G(u)}{\hat G(u)}$:
\begin{align}
    \ex[\xi_1(t) - \xi_2(t)] &= S(t) \int_0^t \left( \frac{-G'(u)\hat{G}(u) + G(u)\hat{G}'(u)}{\hat{G}(u)^2} \right) du = S(t) \int_0^t \frac{d}{du} \left( \frac{-G(u)}{\hat{G}(u)} \right) du \nonumber \\
    &= S(t) \left[ \frac{-G(t)}{\hat{G}(t)} - \frac{-G(0)}{\hat{G}(0)} \right] = S(t) - \frac{S(t)G(t)}{\hat{G}(t)},
\end{align}
using $G(0) = \hat{G}(0) = 1$. Finally, plugging this into $\ex[D(t)] = \frac{S(t)G(t)}{\hat{G}(t)} + S(t) - \frac{S(t)G(t)}{\hat{G}(t)} - S(t)= 0$ gives the desired result for the second case.

\subsection{Influence Function for the CIF}
Our next goal is to derive the influence function for the cumulative incidence functions used in \cref{sec:cr}. In this case, for cause $k$, the survival function $S(t\mid X, Z, W)$ is replaced by the cumulative incidence function
$\mathrm{CIF}_k(t\mid X, Z, W)=P(T_k\le t,\delta=k\mid X, Z, W)$. The identification expression for the CIF potential outcome $\ex[(\mathrm{CIF}_k)_{x_y, W_{x_w}} \mid X = x_z]$ is given by
\begin{align}\label{eq:cif-id}
    \psi(t) =\sum_{z,w}\mathrm{CIF}_k(t \mid x_y, z, w)\,P(w \mid x_w, z)\,P(z \mid x_z).
\end{align}
To obtain the influence function for the CIF identification expression, we proceed very similarly as in the survival function case.
Under $C \ci (T_1,\dots,T_K)\mid X,Z,W$ (see \cref{fig:sfm-cr}), we have that
\begin{align}
    \mathrm{CIF}_k(t\mid X,Z,W)
    =
    \ex\!\left[
        \frac{\mathbbm{1}(M\le t,\delta=k)}{G(M\mid X,Z,W)}
        \,\middle|\,
        X,Z,W
    \right].
\end{align}
Therefore, compared to terms $T_1, T_2, T_3$ in Eqs.~\ref{eq:st-if}-\ref{eq:term-3}, the term $\frac{\mathbbm{1}(M>t)}{G(t \mid X, Z, W)}$ is replaced by $\frac{\mathbbm{1}(M\le t,\delta=k)}{G(M \mid X, Z, W)}$, and $S(t)$ is replaced by $\mathrm{CIF}_k(t)$ throughout, yielding the following influence function
\begin{align}
    \ifc(\psi(t))
    &=
    \frac{\mathbbm{1}(X = x_y)}{P(x_z)}
    \frac{P(x_z \mid Z)}{P(x_w \mid Z)}
    \frac{P(x_w \mid Z,W)}{P(x_y \mid Z,W)}
    \left[
        \frac{\mathbbm{1}(M \le t,\, \delta = k)}{G(M \mid X,Z,W)}
        -
        \mathrm{CIF}_k(t \mid X,Z,W)
    \right]
    \\
    &\quad
    +
    \frac{\mathbbm{1}(X = x_w)}{P(x_z)}
    \frac{P(x_z \mid Z)}{P(x_w \mid Z)}
    \left[
        \mathrm{CIF}_k(t \mid x_y,Z,W)
        -
        \ex\!\left[
            \mathrm{CIF}_k(t \mid x_y,Z,W)
            \,\middle|\,
            X=x_w, Z
        \right]
    \right]
    \\
    &\quad
    +
    \frac{\mathbbm{1}(X = x_z)}{P(x_z)}
    \ex\!\left[
        \mathrm{CIF}_k(t \mid x_y,Z,W)
        \,\middle|\,
        X=x_w, Z
    \right]
    -
    \frac{\mathbbm{1}(X = x_z)}{P(x_z)}\,\psi(t).
\end{align}

\paragraph{Doubly Robust Estimation and Cross-Fitting.}
When performing estimation we use cross-fitting. Data is split into $K$ folds $\mathcal{D}_1, \dots, \mathcal{D}_K$.
Let $\widehat{\mathrm{IF}}^{-(k)}_{\psi(t)}(V_i)$ denote the influence function evaluated for the observed sample $V_i$ using nuisance
estimators fitted on the complement of fold $k$, denoted as $\mathcal{D}_{-k}$. Using one-step debiasing, the cross-fitted estimator for $\psi(t)$ is given by
\begin{align}
    \hat\psi(t)
    &=
    \frac{1}{n}\sum_{k=1}^K \sum_{i \in \mathcal{D}_k}
    \left\{
        \frac{\mathbbm{1}(X_i=x_z)}{\widehat P^{-(k)}(x_z)}\,\hat\psi^{-(k)}(t)
        + \widehat{\mathrm{IF}}^{-(k)}_{\psi(t)}(V_i)
    \right\},
    \label{eq:psi-crossfit}
\end{align}
where $\hat \psi^{-(k)}$ is the standard fold-specific plug-in estimator, given by
\begin{align}
    \hat\psi^{-(k)}(t)
    &=
    \frac{1}{|\mathcal{D}_{-k}|}
    \sum_{i \in \mathcal{D}_{-k}}
    \widehat f^{-(k)}(t\mid x_y,Z_i,W_i)
    \frac{\widehat P^{-(k)}(X_i=x_w\mid Z_i,W_i)}{\widehat P^{-(k)}(X_i=x_w\mid Z_i)}
    \frac{\widehat P^{-(k)}(X_i=x_z\mid Z_i)}{\widehat P^{-(k)}(X_i=x_z)},
    \label{eq:psi-plugin}
\end{align}
where $\widehat f$ is either the estimate of the survival function $S$ or the $\mathrm{CIF}_k$ as appropriate.
The estimator is doubly robust: it is consistent if either
$S(t\mid X,Z,W)$ and $G(t\mid X,Z,W)$ are consistently estimated, or the
propensity scores $P(X\mid Z)$ and $P(X\mid Z,W)$ are consistently estimated.
All nuisance functions are estimated using flexible machine learning methods,
with cross-fitting used to control overfitting bias, known as double machine learning \citep{chernozhukov2018double}. For estimating causal effects, differences of estimators of the form $\hat\psi(t)$ are used, and we obtain confidence intervals assuming asymptotic normality, with the asymptotic variance obtained from the variance of the estimated influence function \citep{chernozhukov2018double}.
\newpage
\section{Copula-Graphic Estimator} \label{appendix:copula-graphic}
In this appendix, we summarize the copula-graphic estimator (CGE) used in \cref{sec:ic} \citep{braekers2005copula}. We consider a single event of interest with time-to-event $T$ and a censoring time $C$ informative for $T$. Following the notation in \cref{sec:classic}, we have that
\begin{align}
    S(t) &= P(T>t), &
    G(t) &= P(C>t), & \\
    S_{\mathrm{all}}(t) &= P(T>t, C>t), &
    H(t,c) &= P(T > t, C > c).
\end{align}
Further, we write $\mathrm{CIF}_T(t)=P(T\leq t, \delta = 1)$ and $\mathrm{CIF}_C(t)=P(C\leq t,\delta =0)$ for the CIFs of the event and censoring, respectively. Note that $S_{\mathrm{all}}(t) = 1 - \mathrm{CIF}_T(t) - \mathrm{CIF}_C(t)$ by definition.
As described in \cref{sec:ic}, we assume an Archimedean copula with generator $\varphi$ (parameterized by Kendall's $\tau$ correlation coefficient) linking the joint survival and marginals as $H(t, c) = \mathcal{C}_\tau (S(t), G(c))$, implying the following relation:
\begin{align} \label{eq:arch-copula-gen}
    \varphi\!\left(S_{\mathrm{all}}(t)\right)
    =
    \varphi\!\left(S(t)\right)
    +
    \varphi\!\left(G(t)\right).
\end{align}
Here, we first describe the copula-graphic estimator in the standard setting.

\paragraph{Classical copula-graphic recursion.}
Let $0=t_0<t_1<\dots<t_m$ be a grid of time points and define the CIF increments
\begin{align}
    \Delta \mathrm{CIF}_{i,T} &= \mathrm{CIF}_T(t_i)-\mathrm{CIF}_T(t_{i-1}), &
    \Delta \mathrm{CIF}_{i,C} &= \mathrm{CIF}_C(t_i)-\mathrm{CIF}_C(t_{i-1}),
\end{align}
and the corresponding decrease of the joint survival
\begin{align}
    \Delta S_{\mathrm{all}}(t_i) = S_{\mathrm{all}}(t_{i-1}) - S_{\mathrm{all}}(t_i)
    =
    \Delta \mathrm{CIF}_{i,T} + \Delta \mathrm{CIF}_{i,C}.
    \label{eq:fall}
\end{align}
The classical CGE assumes that at each $t_i$ only one of the two increments is non-zero, $\Delta \mathrm{CIF}_{i,T}\cdot \Delta \mathrm{CIF}_{i,C} = 0$,
so that the decrease in \cref{eq:fall} can be attributed to exactly one cause at each step. Under this assumption, the recursion proceeds as follows. Suppose $\Delta \mathrm{CIF}_{i,T}>0$ and $\Delta \mathrm{CIF}_{i,C}=0$. Then the event occurs within $(t_{i-1},t_i]$ among those still in the joint risk set, and we have that
\begin{align}
    G(t_i) &= G(t_{i-1}), &
    S_{\mathrm{all}}(t_i) &= S_{\mathrm{all}}(t_{i-1}) - \Delta \mathrm{CIF}_{i,T}.
\end{align}
Plugging into \cref{eq:arch-copula-gen} yields an explicit update for $S(t_i)$:
\begin{align}
    \varphi^{-1}\!\Big(
        \varphi(S_{\mathrm{all}}(t_i)) - \varphi(G(t_{i-1}))
    \Big).
    \label{eq:rec-T}
\end{align}
Analogously, if $\Delta \mathrm{CIF}_{i,C}>0$ and $\Delta \mathrm{CIF}_{i,T}=0$, then $S(t_i)=S(t_{i-1})$ and
\begin{align} \label{eq:rec-C}
    G(t_i) = \varphi^{-1}\!\Big(\varphi(S_{\mathrm{all}}(t_i)) - \varphi(S(t_{i-1}))\Big).
\end{align}
Thus, given $\{S_{\mathrm{all}}(t_i), \mathrm{CIF}_T(t_i), \mathrm{CIF}_C(t_i)\}_{i=0}^m$ (which are estimable from the data even under informative censoring) and the single-jump assumption, the recursion in Eqs.~\ref{eq:rec-T}-\ref{eq:rec-C} identifies the marginal survivals $(S,G)$ for a fixed $\tau$.

\paragraph{Bounds when both increments may be non-zero.}
In our setting, $\mathrm{CIF}_T$ and $\mathrm{CIF}_C$ are available only on a fixed grid of estimates, and we cannot assume $\Delta \mathrm{CIF}_{i,T}\cdot \Delta \mathrm{CIF}_{i,C}=0$. Consequently, \cref{eq:arch-copula-gen} together with $S_{\mathrm{all}}(t_i)$ does not uniquely determine $(S(t_i),G(t_i))$ from $(S(t_{i-1}),G(t_{i-1}))$. \Cref{thm:cge-bounds} below gives sharp one-step lower and upper bounds implied by \cref{eq:arch-copula-gen}, generalizing the classical recursion of \cref{eq:rec-T,eq:rec-C} to the non-disjoint case. To state it, define the marginal decrease candidates
\begin{align}
    H(t_{i-1},t_i) &= S_{\mathrm{all}}(t_{i-1}) - \Delta \mathrm{CIF}_{i,C} = S_{\mathrm{all}}(t_i) + \Delta \mathrm{CIF}_{i,T}, \\
    H(t_i,t_{i-1}) &= S_{\mathrm{all}}(t_{i-1}) - \Delta \mathrm{CIF}_{i,T} = S_{\mathrm{all}}(t_{i})  + \Delta \mathrm{CIF}_{i,C}.
\end{align}
Intuitively, events $T,C$ \emph{compete} for individuals in the risk set in $(t_{i-1},t_i]$. The value $H(t_{i-1},t_i)$ corresponds to the configuration in which the censoring proportion in $(t_{i-1},t_i]$ would be unchanged in the absence of competition from $T$, yielding the upper bound for $G(t_i)$ and lower bound for $S(t_i)$. Conversely, $H(t_i,t_{i-1})$ holds the event proportion fixed in the absence of censoring, yielding the upper bound on $S(t_i)$ and lower bound on $G(t_i)$.

\begin{theorem-restate}[CGE bounds under non-disjoint increments]
\label{thm:cge-bounds-formal}
Let $\varphi$ be an Archimedean copula generator and assume $H(t,c)=\mathcal{C}_\tau(S(t),G(c))$ on $[t_{i-1},t_i]$. Given $(S(t_{i-1}),G(t_{i-1}))$ and the increments $\Delta\mathrm{CIF}_{i,T},\Delta\mathrm{CIF}_{i,C}\geq 0$ with $S_{\mathrm{all}}(t_i)=S_{\mathrm{all}}(t_{i-1})-\Delta\mathrm{CIF}_{i,T}-\Delta\mathrm{CIF}_{i,C}$, the marginals at $t_i$ satisfy $S(t_i)\in[\underline{S}(t_i),\overline{S}(t_i)]$ and $G(t_i)\in[\underline{G}(t_i),\overline{G}(t_i)]$, where
\begin{align}
\overline{G}(t_i) &= \varphi^{-1}\!\big(\varphi(H(t_{i-1},t_i))-\varphi(S(t_{i-1}))\big),\\
\underline{S}(t_i) &= \varphi^{-1}\!\big(\varphi(S_{\mathrm{all}}(t_i)) -\varphi(\overline{G}(t_i)) \big),\\
\overline{S}(t_i) &= \varphi^{-1}\!\big(\varphi(H(t_i,t_{i-1}))-\varphi(G(t_{i-1}))\big),\\
\underline{G}(t_i) &= \varphi^{-1}\!\big(\varphi(S_{\mathrm{all}}(t_i))-\varphi(\overline{S}(t_i)) \big),
\end{align}
with $H(t_{i-1},t_i)=S_{\mathrm{all}}(t_{i-1})-\Delta\mathrm{CIF}_{i,C}$ and $H(t_i,t_{i-1})=S_{\mathrm{all}}(t_{i-1})-\Delta\mathrm{CIF}_{i,T}$. The bounds are sharp, and as $\max_i|t_i-t_{i-1}|\to 0$ they collapse, point-identifying $(S,G)$ for fixed $\tau$.
\end{theorem-restate}

\begin{proof}
By \cref{eq:arch-copula-gen}, $\varphi(H(t,c))=\varphi(S(t))+\varphi(G(c))$ for all $(t,c)$ in the interval. Evaluating at $(t_{i-1},t_i)$ gives
\[
\varphi(G(t_i)) = \varphi(H(t_{i-1},t_i))-\varphi(S(t_{i-1})),
\]
where $G(t_i)$ is maximized when $H(t_{i-1},t_i)$ takes its largest admissible value, namely $S_{\mathrm{all}}(t_{i-1})-\Delta\mathrm{CIF}_{i,C}$ (corresponding to attributing the full censoring increment to $C$ alone in $(t_{i-1},t_i]$, while ascribing none of the joint decrease to $T$ over and above what is forced). This yields $\overline{G}(t_i)$ (since $\varphi$ is strictly decreasing). 
The companion lower bound $\underline{S}(t_i)$ follows by enforcing the additive identity $\varphi(S_{\mathrm{all}}(t_i))=\varphi(S(t_i))+\varphi(G(t_i))$ and substituting $\overline{G}(t_i)$. The pair $(\overline{S}(t_i),\underline{G}(t_i))$ is obtained by the symmetric argument starting from $H(t_i,t_{i-1})$. For sharpness, we note each extremum is attained by the limiting configuration in which one cause's increment is realized first within $(t_{i-1},t_i]$, recovering the classical single-jump recursion of \cref{eq:rec-T,eq:rec-C}. 
For point collapse, observe that when $\Delta\mathrm{CIF}_{i,T} = 0$, update (ii) yields $\overline{S}(t_i) = S(t_{i-1})$ via the additive identity $\varphi(S_{\mathrm{all}}(t_{i-1})) = \varphi(S(t_{i-1})) + \varphi(G(t_{i-1}))$, and update (i) then determines $G(t_i)$ uniquely; the bounds coincide and reduce to the classical recursion of \cref{eq:rec-T,eq:rec-C}. The symmetric argument applies when $\Delta\mathrm{CIF}_{i,C} = 0$. Hence under, any grid refinement finer than the minimum gap between jump times of $\mathrm{CIF}_T$ and $\mathrm{CIF}_C$ (pairwise distinct when $T, C$ are continuously distributed), each interval contains at most one increment, and the bounds collapse.
\end{proof}

In practice, from these bounds we extract the midpoints as our estimates
\begin{align}
    \hat S(t_i) &= \tfrac12\big(\underline{S}(t_i)+\overline{S}(t_i)\big),
    &
    \hat G(t_i) &= \tfrac12\big(\underline{G}(t_i)+\overline{G}(t_i)\big),
\end{align}
and iterate forward in $i$.

\subsection{Uncertainty Envelope} \label{appendix:envelopes}
The above description provides the copula-graphic mapping from cumulative incidence functions to the marginal survival curves.
In the case where uncertainty quantification is available over the CIFs (e.g., when doubly robust estimation is performed), the key challenge is to propagate the uncertainty in the CIFs into the marginal survival curve uncertainty. We now discuss the method for doing so.

\paragraph{Corner evaluation.}
A natural first approximation is to evaluate the copula-graphic estimator at
the four extremal combinations of lower and upper CIF $(1-\alpha)$-confidence intervals,
\[
(\underline{\text{CIF}}_T,\underline{\text{CIF}}_C),\quad
(\underline{\text{CIF}}_T,\overline{\text{CIF}}_C),\quad
(\overline{\text{CIF}}_T,\underline{\text{CIF}}_C),\quad
(\overline{\text{CIF}}_T,\overline{\text{CIF}}_C),
\]
and to take the pointwise minima and maxima of the resulting survival curves
as an uncertainty envelope. While simple and inexpensive, this approach does not provide a  formal guarantee of capturing the true extrema, since the copula-graphic map need not be jointly monotone in the CIF arguments.

\paragraph{Sampling-based envelope.}
To obtain a more reliable envelope, we augment the corner evaluation with a
sampling step over the admissible CIF region circumscribed by its $(1-\alpha)$-confidence interval. After estimating lower and upper bounds for $\text{CIF}_T$ and $\text{CIF}_C$, we repeatedly draw CIF trajectories uniformly within these bands on the time grid, and ensure each draw is a valid CIF by enforcing
(i) non-negativity,
(ii) non-negative increments, and
(iii) the pointwise constraint $\text{CIF}_T(t)+\text{CIF}_C(t)\le 1$.
For each sampled pair, we apply the copula-graphic recursion to obtain a
candidate survival curve. The final uncertainty envelope is then defined by the
pointwise minima and maxima over the survival curves obtained from the four
corners together with the sampled trajectories, and these confidence intervals are reported in practice (see \cref{fig:readm-effects}).
\newpage
\section{ANZICS Acknowledgement} \label{appendix:anzics-ack}
The authors acknowledge the Australian and New Zealand Intensive Care Society (ANZICS) Centre for Outcomes and Resource Evaluation (CORE) for providing the data used in the current study. The authors and the management committee of ANZICS CORE would like to thank clinicians, data collectors and researchers at the following contributing sites:
\begin{longtable}{|c|}
    \hline
    \endfirsthead

    \hline
    \endhead

    \hline
    \endfoot

    \hline
    \endlastfoot

    Albury Wodonga Health ICU \\
    \hline
    Alfred Hospital ICU \\
    \hline
    Alice Springs Hospital ICU \\
    \hline
    Allamanda Private Hospital ICU \\
    \hline
    Angliss Hospital ICU \\
    \hline
    Armadale Health Service ICU \\
    \hline
    Ashford Community Hospital ICU \\
    \hline
    Auckland City Hospital CV ICU \\
    \hline
    Auckland City Hospital DCCM \\
    \hline
    Austin Hospital ICU \\
    \hline
    Ballarat Health Services ICU \\
    \hline
    Bankstown-Lidcombe Hospital ICU \\
    \hline
    Bathurst Base Hospital ICU \\
    \hline
    Bendigo Health Care Group ICU \\
    \hline
    Blacktown Hospital ICU \\
    \hline
    Bowral Hospital HDU \\
    \hline
    Box Hill Hospital ICU \\
    \hline
    Braemar Hospital SCU \\
    \hline
    Brisbane Private Hospital ICU \\
    \hline
    Brisbane Waters Private Hospital ICU \\
    \hline
    Broken Hill Base Hospital \& Health Services ICU \\
    \hline
    Buderim Private Hospital ICU \\
    \hline
    Bunbury Regional Hospital ICU \\
    \hline
    Bundaberg Base Hospital ICU \\
    \hline
    Caboolture Hospital ICU \\
    \hline
    Cabrini Hospital ICU \\
    \hline
    Cairns Hospital ICU \\
    \hline
    Calvary Adelaide Hospital ICU \\
    \hline
    Calvary Bruce Private Hospital HDU \\
    \hline
    Calvary Hospital (Canberra) ICU \\
    \hline
    Calvary Hospital (Lenah Valley) ICU \\
    \hline
    Calvary John James Hospital ICU \\
    \hline
    Calvary Mater Newcastle ICU \\
    \hline
    Calvary North Adelaide Hospital ICU \\
    \hline
    Campbelltown Hospital ICU \\
    \hline
    Canberra Hospital ICU \\
    \hline
    Casey Hospital ICU \\
    \hline
    Central Gippsland Health Service (Sale) ICU \\
    \hline
    Christchurch Hospital ICU \\
    \hline
    Coffs Harbour Health Campus ICU \\
    \hline
    Concord Hospital (Sydney) ICU \\
    \hline
    Dandenong Hospital ICU \\
    \hline
    Dubbo Base Hospital ICU \\
    \hline
    Dunedin Hospital ICU \\
    \hline
    Echuca Regional Hospital HDU \\
    \hline
    Epworth Eastern Private Hospital ICU \\
    \hline
    Epworth Freemasons Hospital ICU \\
    \hline
    Epworth Geelong ICU \\
    \hline
    Epworth Hospital (Richmond) ICU \\
    \hline
    Fairfield Hospital ICU \\
    \hline
    Fiona Stanley Hospital ICU \\
    \hline
    Flinders Medical Centre ICU \\
    \hline
    Flinders Private Hospital ICU \\
    \hline
    Footscray Hospital ICU \\
    \hline
    Frankston Hospital ICU \\
    \hline
    Fremantle Hospital ICU \\
    \hline
    Gold Coast Private Hospital ICU \\
    \hline
    Gold Coast University Hospital ICU \\
    \hline
    Gosford Hospital ICU \\
    \hline
    Gosford Private Hospital ICU \\
    \hline
    Goulburn Base Hospital ICU \\
    \hline
    Goulburn Valley Health ICU \\
    \hline
    Grafton Base Hospital ICU \\
    \hline
    Greenslopes Private Hospital ICU \\
    \hline
    Griffith Base Hospital ICU \\
    \hline
    Hawkes Bay Hospital ICU \\
    \hline
    Hervey Bay Hospital ICU \\
    \hline
    Hollywood Private Hospital ICU \\
    \hline
    Holmesglen Private Hospital ICU \\
    \hline
    Hornsby Ku-ring-gai Hospital ICU \\
    \hline
    Hurstville Private Hospital ICU \\
    \hline
    Hutt Hospital ICU \\
    \hline
    Ipswich Hospital ICU \\
    \hline
    John Fawkner Hospital ICU \\
    \hline
    John Hunter Hospital ICU \\
    \hline
    Joondalup Health Campus ICU \\
    \hline
    Kareena Private Hospital ICU \\
    \hline
    Knox Private Hospital ICU \\
    \hline
    Latrobe Regional Hospital ICU \\
    \hline
    Launceston General Hospital ICU \\
    \hline
    Lingard Private Hospital ICU \\
    \hline
    Lismore Base Hospital ICU \\
    \hline
    Liverpool Hospital ICU \\
    \hline
    Logan Hospital ICU \\
    \hline
    Lyell McEwin Hospital ICU \\
    \hline
    Mackay Base Hospital ICU \\
    \hline
    Macquarie University Private Hospital ICU \\
    \hline
    Maitland Hospital ICU \\
    \hline
    Maitland Private Hospital ICU \\
    \hline
    Manly Hospital \& Community Health ICU \\
    \hline
    Manning Rural Referral Hospital ICU \\
    \hline
    Maroondah Hospital ICU \\
    \hline
    Mater Adults Hospital (Brisbane) ICU \\
    \hline
    Mater Health Services North Queensland ICU \\
    \hline
    Mater Private Hospital (Brisbane) ICU \\
    \hline
    Mater Private Hospital (Sydney) ICU \\
    \hline
    Melbourne Private Hospital ICU \\
    \hline
    Middlemore Hospital ICU \\
    \hline
    Mildura Base Public Hospital ICU \\
    \hline
    Modbury Public Hospital ICU \\
    \hline
    Monash Medical Centre-Clayton Campus ICU \\
    \hline
    Mount Hospital ICU \\
    \hline
    Mount Isa Hospital ICU \\
    \hline
    Mulgrave Private Hospital ICU \\
    \hline
    Nambour General Hospital ICU \\
    \hline
    National Capital Private Hospital ICU \\
    \hline
    Nelson Hospital ICU \\
    \hline
    Nepean Hospital ICU \\
    \hline
    Nepean Private Hospital ICU \\
    \hline
    Newcastle Private Hospital ICU \\
    \hline
    Noosa Hospital ICU \\
    \hline
    North Shore Hospital ICU \\
    \hline
    North Shore Private Hospital ICU \\
    \hline
    North West Private Hospital ICU \\
    \hline
    North West Regional Hospital (Burnie) ICU \\
    \hline
    Northeast Health Wangaratta ICU \\
    \hline
    Northern Beaches Hospital ICU \\
    \hline
    Norwest Private Hospital ICU \\
    \hline
    Orange Base Hospital ICU \\
    \hline
    Peninsula Private Hospital ICU \\
    \hline
    Peter MacCallum Cancer Institute ICU \\
    \hline
    Pindara Private Hospital ICU \\
    \hline
    Port Macquarie Base Hospital ICU \\
    \hline
    Prince of Wales Hospital (Hong Kong) ICU \\
    \hline
    Prince of Wales Hospital (Sydney) ICU \\
    \hline
    Prince of Wales Private Hospital (Sydney) ICU \\
    \hline
    Princess Alexandra Hospital ICU \\
    \hline
    Queen Elizabeth II Jubilee Hospital ICU \\
    \hline
    Redcliffe Hospital ICU \\
    \hline
    Robina Hospital ICU \\
    \hline
    Rockhampton Hospital ICU \\
    \hline
    Rockingham General Hospital ICU \\
    \hline
    Rotorua Hospital ICU \\
    \hline
    Royal Adelaide Hospital ICU \\
    \hline
    Royal Brisbane and Women's Hospital ICU \\
    \hline
    Royal Darwin Hospital ICU \\
    \hline
    Royal Hobart Hospital ICU \\
    \hline
    Royal Melbourne Hospital ICU \\
    \hline
    Royal North Shore Hospital ICU \\
    \hline
    Royal Perth Hospital ICU \\
    \hline
    Royal Prince Alfred Hospital ICU \\
    \hline
    Ryde Hospital and Community Health Services ICU \\
    \hline
    Shoalhaven Hospital ICU \\
    \hline
    Sir Charles Gairdner Hospital ICU \\
    \hline
    South East Regional Hospital ICU \\
    \hline
    South West Healthcare (Warrnambool) ICU \\
    \hline
    Southern Cross Hospital (Hamilton) ICU \\
    \hline
    Southern Cross Hospital (Wellington) ICU \\
    \hline
    St Andrew's Hospital Toowoomba ICU \\
    \hline
    St Andrew's Private Hospital (Ipswich) ICU \\
    \hline
    St Andrew's War Memorial Hospital ICU \\
    \hline
    St George Hospital (Sydney) CICU \\
    \hline
    St George Hospital (Sydney) ICU \\
    \hline
    St George Hospital (Sydney) ICU2 \\
    \hline
    St George Private Hospital (Sydney) ICU \\
    \hline
    St John Of God Health Care (Subiaco) ICU \\
    \hline
    St John Of God Hospital (Ballarat) ICU \\
    \hline
    St John of God Hospital (Bendigo) ICU \\
    \hline
    St John of God Hospital (Berwick) ICU \\
    \hline
    St John Of God Hospital (Geelong) ICU \\
    \hline
    St John of God Midland Public \& Private ICU \\
    \hline
    St Vincent’s Private Hospital Northside ICU \\
    \hline
    St Vincent's Hospital (Melbourne) ICU \\
    \hline
    St Vincent's Hospital (Sydney) ICU \\
    \hline
    St Vincent's Hospital (Toowoomba) ICU \\
    \hline
    St Vincent's Private Hospital (Sydney) ICU \\
    \hline
    St Vincent's Private Hospital Fitzroy ICU \\
    \hline
    Sunnybank Hospital ICU \\
    \hline
    Sunshine Coast University Hospital ICU \\
    \hline
    Sunshine Coast University Private Hospital ICU \\
    \hline
    Sunshine Hospital ICU \\
    \hline
    Sutherland Hospital \& Community Health Services ICU \\
    \hline
    Sydney Adventist Hospital ICU \\
    \hline
    Tamworth Base Hospital ICU \\
    \hline
    Taranaki Health ICU \\
    \hline
    Tauranga Hospital ICU \\
    \hline
    The Bays Hospital ICU \\
    \hline
    The Chris O’Brien Lifehouse ICU \\
    \hline
    The Memorial Hospital (Adelaide) ICU \\
    \hline
    The Northern Hospital ICU \\
    \hline
    The Prince Charles Hospital ICU \\
    \hline
    The Queen Elizabeth (Adelaide) ICU \\
    \hline
    The Wesley Hospital ICU \\
    \hline
    Timaru Hospital ICU \\
    \hline
    Toowoomba Hospital ICU \\
    \hline
    Townsville University Hospital ICU \\
    \hline
    Tweed Heads District Hospital ICU \\
    \hline
    University Hospital Geelong ICU \\
    \hline
    Wagga Wagga Base Hospital \& District Health ICU \\
    \hline
    Waikato Hospital ICU \\
    \hline
    Wairau Hospital ICU \\
    \hline
    Warringal Private Hospital ICU \\
    \hline
    Wellington Hospital ICU \\
    \hline
    Werribee Mercy Hospital ICU \\
    \hline
    Western District Health Service (Hamilton) ICU \\
    \hline
    Western Hospital (SA) ICU \\
    \hline
    Western Private Hospital ICU \\
    \hline
    Westmead Hospital ICU \\
    \hline
    Westmead Private Hospital ICU \\
    \hline
    Whangarei Area Hospital - Northland Health Ltd ICU \\
    \hline
    Wimmera Health Care Group (Horsham) ICU \\
    \hline
    Wollongong Hospital ICU \\
    \hline
    Women's and Children's Hospital PICU \\
    \hline
    Wyong Hospital ICU \\
\end{longtable}

\ifnum\treportflag=0
    \newpage
    \section*{NeurIPS Paper Checklist}

\begin{enumerate}

\item {\bf Claims}
    \item[] Question: Do the main claims made in the abstract and introduction accurately reflect the paper's contributions and scope?
    \item[] Answer: \answerYes{}
    \item[] Justification: The four contributions listed in the introduction are each substantiated: the Causal Reduction Theorem (\cref{thm:causal-reduction}) and framework instantiation under non-informative censoring, competing risks, and informative censoring (\cref{sec:classic,sec:cr,sec:ic}); identification results (\cref{prop:id}); two new technical results, for doubly-robust estimation (\cref{thm:dr-est}) and sharp copula-graphic bounds (\cref{thm:cge-bounds}); and the empirical case study on post-ICU racial disparities (\cref{sec:experiments}).

\item {\bf Limitations}
    \item[] Question: Does the paper discuss the limitations of the work performed by the authors?
    \item[] Answer: \answerYes{}
    \item[] Justification: 
    The paper contains an explicit Limitations section at the end, acknowledging the strength of the causal assumptions used, and the need to justify such assumptions on a case-specific basis. We further identify sensitivity analysis for such unobserved confounding in mediation settings as an important direction for future work.
    Limitations are also discussed within the case study section, where we acknowledge potential sources of unmeasured confounding (genetic variation, geographic remoteness).

\item {\bf Theory assumptions and proofs}
    \item[] Question: For each theoretical result, does the paper provide the full set of assumptions and a complete (and correct) proof?
    \item[] Answer: \answerYes{}
    \item[] Justification: All assumptions are stated alongside the theoretical results. Proofs for \cref{thm:causal-reduction}, \cref{prop:id}, and \cref{prop:model-est} are provided in \cref{appendix:proofs}. The proof of \cref{thm:dr-est} is provided in \cref{appendix:dr}, and the proof of \cref{thm:cge-bounds} is provided in \cref{appendix:copula-graphic}.

    \item {\bf Experimental result reproducibility}
    \item[] Question: Does the paper fully disclose all the information needed to reproduce the main experimental results of the paper to the extent that it affects the main claims and/or conclusions of the paper (regardless of whether the code and data are provided or not)?
    \item[] Answer: \answerYes{}
    \item[] Justification: All data sources (ANZICS APD), variable definitions, estimation procedures (random survival forests, xgboost-based propensity estimation, cross-fitting), and sensitivity analysis parameters are described in the main text. Full code is provided in an anonymized repository.

\item {\bf Open access to data and code}
    \item[] Question: Does the paper provide open access to the data and code, with sufficient instructions to faithfully reproduce the main experimental results, as described in supplemental material?
    \item[] Answer: \answerYes{}
    \item[] Justification: The code is publicly available in an anonymized repository. The ANZICS APD dataset is available to researchers upon institutional application; the paper describes all preprocessing steps and variable definitions needed to reproduce the analyses.

\item {\bf Experimental setting/details}
    \item[] Question: Does the paper specify all the training and test details (e.g., data splits, hyperparameters, how they were chosen, type of optimizer) necessary to understand the results?
    \item[] Answer: \answerYes{}
    \item[] Justification: All experimental details (ANZICS APD data, variable groupings, random survival forest estimation, xgboost propensity models, cross-fitting procedure, and sensitivity parameter choices for the informative censoring analysis) are described in the main text.

\item {\bf Experiment statistical significance}
    \item[] Question: Does the paper report error bars suitably and correctly defined or other appropriate information about the statistical significance of the experiments?
    \item[] Answer: \answerYes{}
    \item[] Justification: 95\% confidence intervals are reported for all causal effect estimates, obtained under asymptotic normality of the cross-fitted doubly-robust estimator. For the sensitivity analysis under informative censoring, maximal and minimal 95\% confidence intervals across sensitivity parameter values are displayed.

\item {\bf Experiments compute resources}
    \item[] Question: For each experiment, does the paper provide sufficient information on the computer resources (type of compute workers, memory, time of execution) needed to reproduce the experiments?
    \item[] Answer: \answerYes{}
    \item[] Justification: Information is provided at the start of the supplementary material. All experiments were run on a MacBook Pro with M3 Pro chip and 36\,GB RAM (macOS 26.2), with total compute under 24 hours.
    
\item {\bf Code of ethics}
    \item[] Question: Does the research conducted in the paper conform, in every respect, with the NeurIPS Code of Ethics \url{https://neurips.cc/public/EthicsGuidelines}?
    \item[] Answer: \answerYes{}
    \item[] Justification: The research complies with the NeurIPS Code of Ethics.

\item {\bf Broader impacts}
    \item[] Question: Does the paper discuss both potential positive societal impacts and negative societal impacts of the work performed?
    \item[] Answer: \answerYes{}
    \item[] Justification: A Broader Impacts Statement is included at the end of the main paper. It highlights the positive impact of enabling more principled analyses of disparities in survival settings, contributing to identifying and reducing unfair outcomes in applied domains. We do not foresee direct ethical risks or potential harms from this work.
    
\item {\bf Safeguards}
    \item[] Question: Does the paper describe safeguards that have been put in place for responsible release of data or models that have a high risk for misuse (e.g., pre-trained language models, image generators, or scraped datasets)?
    \item[] Answer: \answerNA{}
    \item[] Justification: The paper introduces a bias detection methodology for existing models. It does not release new models or datasets, posing no such risks.

\item {\bf Licenses for existing assets}
    \item[] Question: Are the creators or original owners of assets (e.g., code, data, models), used in the paper, properly credited and are the license and terms of use explicitly mentioned and properly respected?
    \item[] Answer: \answerYes{}
    \item[] Justification: All assets used are properly cited; all are publicly available under open or permissive licenses.

\item {\bf New assets}
    \item[] Question: Are new assets introduced in the paper well documented and is the documentation provided alongside the assets?
    \item[] Answer: \answerYes{}
    \item[] Justification: The released code implementing the methodology is documented with a README explaining the setup and reproduction steps.

\item {\bf Crowdsourcing and research with human subjects}
    \item[] Question: For crowdsourcing experiments and research with human subjects, does the paper include the full text of instructions given to participants and screenshots, if applicable, as well as details about compensation (if any)?
    \item[] Answer: \answerNA{}
    \item[] Justification: No crowdsourcing or human subjects involved.

\item {\bf Institutional review board (IRB) approvals or equivalent for research with human subjects}
    \item[] Question: Does the paper describe potential risks incurred by study participants, whether such risks were disclosed to the subjects, and whether Institutional Review Board (IRB) approvals (or an equivalent approval/review based on the requirements of your country or institution) were obtained?
    \item[] Answer: \answerYes{}{}
    \item[] Justification: The analysis in Sec.~\ref{sec:experiments} was approved by our local Ethics Committee / Institutional Review Board (IRB). Full institutional details and project numbers are redacted for double-blind review and will be provided in the camera-ready version.

\item {\bf Declaration of LLM usage}
    \item[] Question: Does the paper describe the usage of LLMs if it is an important, original, or non-standard component of the core methods in this research? Note that if the LLM is used only for writing, editing, or formatting purposes and does \emph{not} impact the core methodology, scientific rigor, or originality of the research, declaration is not required.
    \item[] Answer: \answerNA{}
    \item[] Justification: LLMs were used for text polishing, but were not used to generate the paper's core methodology or developments. Accordingly, no declaration is provided.

\end{enumerate}
\fi

\end{document}